\documentclass{article}
\usepackage{iclr2027_conference,times}
\usepackage[T1]{fontenc}

\usepackage{amsmath,amsfonts,bm}

\def\eqref#1{equation~\ref{#1}}

\def\1{\bm{1}}

\DeclareMathAlphabet{\mathsfit}{\encodingdefault}{\sfdefault}{m}{sl}
\SetMathAlphabet{\mathsfit}{bold}{\encodingdefault}{\sfdefault}{bx}{n}

\usepackage{amsmath,amssymb}
\usepackage{amsthm}
\usepackage{booktabs}
\usepackage{algorithm}
\usepackage{algpseudocode}
\usepackage{graphicx}
\usepackage{placeins}
\usepackage{subcaption}
\usepackage{wrapfig}
\usepackage{needspace}
\usepackage{multirow}
\usepackage{array}
\usepackage{xcolor}
\usepackage{tikz}
\usetikzlibrary{arrows.meta,positioning,fit,calc}
\usepackage{hyperref}
\usepackage{xurl}
\usepackage{url}
\hypersetup{hypertexnames=false}


\newtheorem{proposition}{Proposition}

\newcommand{\model}{\mbox{\textsc{RetroGEF}}}
\newcommand{\tableformat}{%
  \scriptsize
  \setlength{\tabcolsep}{2pt}%
  \renewcommand{\arraystretch}{1.04}%
  \setlength{\aboverulesep}{0.35ex}%
  \setlength{\belowrulesep}{0.35ex}%
}

\title{\model{}: Dynamic Graph Edit Flow for Single-Step Retrosynthesis}

\author{%
Xiaozhuang Song\textsuperscript{1,2} \quad
Xuemin Chen\textsuperscript{1,2} \quad
Xinjian Zhao\textsuperscript{1,2} \\
\textbf{Yaoyao Xu\textsuperscript{1,2}} \quad
\textbf{Tianshu Yu\textsuperscript{1,2}}\thanks{Corresponding author.} \\
\textnormal{\textsuperscript{1}The Chinese University of Hong Kong, Shenzhen} \\
\textnormal{\textsuperscript{2}Shanghai AI Laboratory} \\
\textnormal{\texttt{\{xiaozhuangsong1,xueminchen,xinjianzhao1\}@link.cuhk.edu.cn}} \\
\textnormal{\texttt{yaoyaoxu@link.cuhk.edu.cn \quad yutianshu@cuhk.edu.cn}}
}

\iclrfinalcopy

\begin{document}

\maketitle
\lhead{Preprint}

\begin{abstract}
Retrosynthesis enables the discovery of viable synthetic routes to target
molecules. It plays a central role in modern drug discovery and materials
design. Retrosynthesis involves molecular graph transformations that can
change both connectivity and graph size. These transformations may introduce
reactant components absent from the target while revising the product-derived
structure. To model these transformations, we propose \model{}, a flow-based
generative model for single-step retrosynthesis. Starting from the target
molecule, it constructs possible reactants by adding atoms and changing bonds
in the molecular graph. \model{} models molecular transformations and changes
in graph size within the same generative process, rather than relying on a
fixed-size graph canvas. It learns this process directly from product--reactant
pairs without requiring a prescribed edit order. Experiments on representative
retrosynthesis benchmarks demonstrate that \model{} achieves state-of-the-art
performance.

\end{abstract}

\section{Introduction}
\label{sec:introduction}

Single-step retrosynthesis predicts the precursor molecules needed to obtain a
desired product, providing the individual transformations used in synthesis
planning \citep{segler2018planning,chen2020retrostar,song2026aot}.
Recovering these precursors requires more than changing bonds in the product:
the reactant set may contain several molecules and additional atoms whose
number and connectivity are not known in advance
\citep{yao2023nag2g,zhong2023graph2edits}. A model must both construct missing
structure and revise the structure it inherits. The central challenge is to
learn this variable-size construction process from reaction records that
specify the product and reactants, but not how one graph is built from the other.

Existing representations organize this problem differently. Sequence models
serialize reactants, graph-edit models expose structural transformations
\citep{zhong2022rsmiles,sacha2021megan,zhong2023graph2edits,deng2025rsgpt},
and diffusion or flow models learn stochastic paths between reaction graphs
\citep{igashov2024retrobridge,yadav2025retrosynflow,wang2026retrodit}.
A fixed graph canvas reserves capacity before generation, whereas a growing
graph must decide how each new atom enters and connects to the existing
structure. Atom creation and connectivity are coupled: a new atom becomes an
attachment site, and its first bond changes the context for subsequent edits.
This motivates a representation in which growth and structural revision act
on the same evolving graph.

Learning these edits introduces a second ambiguity. Multiple construction
paths can reach the same reactant set, yet reaction records provide no preferred
order \citep{sacha2021megan,zhong2023graph2edits}. Nor does a partial graph
always identify which missing reactant atom a newly added atom represents:
different correspondences can yield an identical structural update. Supervision
should accommodate these alternatives rather than treat an arbitrary edit
sequence or atom correspondence as the unique answer. The learning objective
must therefore connect paired reaction graphs through the edits that the model
will actually use during generation.

We propose \model{}, a continuous-time graph jump process built around
\emph{complete graph edits}. Each edit specifies a complete structural change:
adding an attached atom determines its attachment site, attributes, and first
bond together. The resulting graph immediately becomes the context for the next
edit. Additional atoms can extend the new structure, while bond and attribute
edits can revise both inherited and generated parts. Representing these changes
relative to the product lets the state grow through construction rather than
occupy a preallocated reactant canvas.

To learn this process, we construct a stochastic bridge from each product to
its recorded reactant graph. At an intermediate graph, the bridge assigns rates
to eligible edits, providing supervision for both which change to make and
when to make it. Training samples different edit orders and aggregates rates
when hidden atom correspondences yield the same observable update. Generator
matching learns these rates from the sampled bridge states
\citep{holderrieth2024generatormatching}, aligning paired-graph supervision
with the complete edits used at inference.

Our contributions are threefold:
\begin{itemize}
  \item A variable-size graph jump process that couples atom creation and
  initial connectivity through complete edits, allowing growth and revision
  within a single evolving state.
  \item A paired-graph bridge that supervises edit rates across multiple
  construction orders and combines equivalent atom correspondences into
  observable transitions.
  \item An empirical evaluation on USPTO-Full and USPTO-50K demonstrating
  improved top-$k$ accuracy under predicted and reference reaction-center
  conditioning, with controlled ablations and analyses of graph growth,
  structural revision, and reaction classes.
\end{itemize}

\section{Related Work}
\label{sec:related}

\paragraph{Representations for retrosynthesis.}
Retrosynthesis methods differ in the structural decisions they expose.
Template-based models select transformations by similarity or learned representations
\citep{coley2017retrosim,dai2019gln,chen2021localretro}.
Sequence models generate SMILES \citep{weininger1988smiles,liu2017seq2seq}
with augmentation, alignment, or pretraining
\citep{tetko2020augmented,zhong2022rsmiles,deng2025rsgpt}.
Graph-to-sequence models retain structured inputs
\citep{seo2021gta,tu2022graph2smiles}.
Two-stage models identify centers or edits before completing synthons
\citep{shi2020g2gs,yan2020retroxpert,somnath2021graphretro,chen2023g2retro}.
Graph-to-graph translation \citep{lin2023g2gt,yao2023nag2g} and autoregressive
editing \citep{sacha2021megan,zhong2023graph2edits} expose molecular structure
and intermediate states. We instead learn rates of complete structural edits
across multiple product-to-reactant paths.

\paragraph{Graph generation and variable-size structure.}
Autoregressive generators make successive node and edge decisions
\citep{you2018graphrnn,shi2020graphaf,luo2021graphdf}, whereas one-shot models
generate graphs jointly \citep{martinkus2022spectre}. Score-based and
discrete-state models evolve their representations
\citep{niu2020scoregraph,jo2022gdss,vignac2023digress,qin2025defog}.
Variable-size approaches include
trans-dimensional jump diffusion \citep{campbell2023transdimensional} and
sequence insertion/deletion flows \citep{havasi2025editflows}.
Retrosynthetic growth must additionally specify connectivity relative to the
product \citep{yao2023nag2g,laabid2025equivariant}. Our attached atom edit
addresses this requirement by coupling atom creation to its first bond. This
keeps the evolving structure anchored to the product while allowing its size
and component structure to be determined during generation.

\paragraph{Discrete dynamics and reaction-conditioned bridges.}
Discrete generation learns categorical transitions \citep{austin2021d3pm},
continuous-time jump rates \citep{campbell2022ctdd}, or reverse probability
ratios \citep{lou2024sedd}. Flow matching learns from conditional paths in
continuous \citep{lipman2023flowmatching} and discrete spaces
\citep{campbell2024discreteflows,gat2024discreteflow}; Generator Matching
extends this principle to Markov generators \citep{holderrieth2024generatormatching}.
For retrosynthesis, RetroBridge uses Markov bridges and RetroDiff uses staged
distribution interpolation
\citep{igashov2024retrobridge,wang2025retrodiff}.
Retro SynFlow and RetroDiT apply discrete flow matching, the latter with
reaction-center guidance
\citep{yadav2025retrosynflow,wang2026retrodit}.
Our bridge supervises complete edits, combining equivalent hidden atom
correspondences while learning transition intensity and edit probabilities.
The resulting construction order remains latent, while each observable edit
is defined on the current graph and can be revised by later transitions.
This perspective connects variable-size graph generation with the structural
revisions required by retrosynthesis, which we develop in the next section.

\section{\model{}}
\label{sec:method}

\model{} (Retrosynthetic Graph Edit Flow) learns a continuous-time process
of complete graph edits from paired product and reactant graphs.
Figure~\ref{fig:graph-generator} summarizes training and generation.

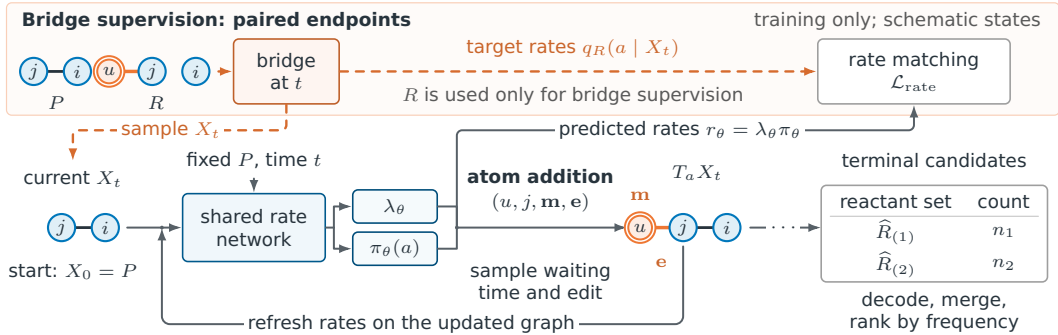
\begin{figure}[tb!]
\centering
\resizebox{\textwidth}{!}{
\begingroup%
\renewcommand{\sfdefault}{DejaVuSans-TLF}%
\definecolor{gefBlue}{HTML}{0077BB}%
\definecolor{gefOrange}{HTML}{EE7733}%
\definecolor{gefInk}{HTML}{27333C}%
\begin{tikzpicture}[
  x=1cm,y=1cm,
  font=\sffamily\fontsize{7}{8.4}\selectfont,
  text=gefInk,line cap=round,line join=round,
  gef label/.style={inner sep=1.3pt,align=center},
  gef anchor/.style={circle,draw=gefBlue,line width=0.8pt,
    fill=gefBlue!12,minimum size=3.6mm,inner sep=0pt},
  gef new/.style={circle,draw=gefOrange,line width=0.75pt,
    double=white,double distance=0.55pt,fill=gefOrange!13,
    minimum size=3.6mm,inner sep=0pt},
  gef bond/.style={draw=gefInk,line width=1.05pt},
  gef firstbond/.style={draw=gefOrange,line width=1.3pt},
  gef wire/.style={draw=gefInk!78,line width=0.75pt},
  gef flow/.style={gef wire,-{Latex[length=1.7mm,width=1.2mm]}},
  gef teacher/.style={draw=gefOrange!90!black,line width=0.9pt,
    dashed,-{Latex[length=1.7mm,width=1.2mm]}},
  gef box/.style={draw=black!38,line width=0.6pt,
    rounded corners=2pt,fill=white,align=center,inner sep=2pt},
  gef head/.style={gef box,draw=gefBlue!65!black,fill=gefBlue!4,
    minimum width=1.12cm,minimum height=0.44cm},
  pics/gefstate/.style args={#1}{code={
    \ifnum#1<2
      \draw[gef bond] (-0.15,0)--(0.43,0);
    \fi
    \ifnum#1>0
      \draw[gef firstbond] (-0.72,0)--(-0.15,0);
    \fi
    \node[gef anchor] at (-0.15,0) {\(j\)};
    \node[gef anchor] at (0.43,0) {\(i\)};
    \ifnum#1>0
      \node[gef new] at (-0.72,0) {\(u\)};
    \fi
  }}
]
\path[use as bounding box] (0,0) rectangle (14,4.7);

\draw[draw=gefOrange!35,fill=gefOrange!3,line width=0.45pt,
  rounded corners=2pt] (0.02,3.19) rectangle (13.98,4.68);
\node[gef label,anchor=west,font=\sffamily\bfseries\fontsize{7}{8.4}\selectfont]
  at (0.17,4.44) {Bridge supervision: paired endpoints};
\node[gef label,anchor=east,text=black!65] at (13.79,4.44)
  {training only; schematic states};
\pic at (0.53,3.77) {gefstate=0};
\pic at (2.10,3.77) {gefstate=2};
\node[gef label] at (0.68,3.35) {\(P\)};
\node[gef label] at (2.00,3.35) {\(R\)};

\node[gef box,draw=gefOrange!80!black,fill=gefOrange!7,
  minimum width=1.40cm,minimum height=0.86cm]
  (gefbridge) at (3.73,3.77) {bridge\\at \(t\)};
\draw[gef teacher] (2.83,3.77)--(gefbridge.west);
\node[gef box,minimum width=2.55cm,minimum height=0.86cm]
  (gefloss) at (12.05,3.77)
  {rate matching\\\(\mathcal L_{\mathrm{rate}}\)};
\draw[gef teacher] (gefbridge.east)--(gefloss.west);
\node[gef label,text=gefOrange!85!black] at (7.53,4.08)
  {target rates \(q_R(a\mid X_t)\)};
\node[gef label,text=black!65] at (7.53,3.45)
  {\(R\) is used only for bridge supervision};

\draw[gef teacher,rounded corners=2pt]
  (gefbridge.south)--(3.73,2.99)--(0.90,2.99)--(0.90,2.58);
\node[gef label,fill=white,text=gefOrange!85!black] at (2.20,2.99)
  {sample \(X_t\)};
\node[gef label] at (0.90,2.35) {current \(X_t\)};
\pic at (0.90,1.70) {gefstate=0};
\node[gef label] at (0.90,1.12) {start: \(X_0=P\)};

\node[gef box,draw=gefBlue!75!black,fill=gefBlue!6,
  text width=1.75cm,minimum height=0.90cm]
  (gefnetwork) at (3.30,1.70) {shared rate\\network};
\draw[gef flow] (1.62,1.70)--(gefnetwork.west);
\node[gef label] at (3.30,2.59) {fixed \(P\), time \(t\)};
\draw[gef flow] (3.30,2.35)--(gefnetwork.north);

\node[gef head] (geflambda) at (5.17,1.98) {\(\lambda_\theta\)};
\node[gef head] (gefpi) at (5.17,1.42) {\(\pi_\theta(a)\)};
\draw[gef wire] (gefnetwork.east)--(4.36,1.70);
\draw[gef flow] (4.36,1.70)|-(geflambda.west);
\draw[gef flow] (4.36,1.70)|-(gefpi.west);
\draw[gef wire] (geflambda.east)-|(5.99,1.70);
\draw[gef wire] (gefpi.east)-|(5.99,1.70);
\fill[gefInk!78] (5.99,1.70) circle (0.025);

\draw[gef flow,rounded corners=2pt]
  (5.99,1.70)--(5.99,2.99)--(12.05,2.99)--(gefloss.south);
\node[gef label,fill=white] at (8.95,2.99)
  {predicted rates \(r_\theta=\lambda_\theta\pi_\theta\)};

\draw[gef flow] (5.99,1.70)--(8.25,1.70);
\node[gef label,font=\sffamily\bfseries\fontsize{7}{8.4}\selectfont]
  at (7.10,2.39) {atom addition};
\node[gef label] at (7.10,2.03) {\((u,j,\mathbf m,\mathbf e)\)};
\node[gef label] at (7.10,0.99) {sample waiting\\time and edit};
\pic at (9.14,1.70) {gefstate=1};
\node[gef label] at (9.18,2.43) {\(T_aX_t\)};
\node[gef label,text=gefOrange!85!black] at (8.42,2.12) {\(\mathbf m\)};
\node[gef label,text=gefOrange!85!black] at (8.70,1.23) {\(\mathbf e\)};

\draw[gef flow] (9.88,1.70)--(10.82,1.70);
\node[gef label,fill=white] at (10.32,1.70) {\(\cdots\)};
\node[gef label] at (12.32,2.60) {terminal candidates};
\draw[gef box] (10.82,0.98) rectangle (13.82,2.26);
\node[gef label] at (11.81,2.05) {reactant set};
\node[gef label] at (13.24,2.05) {count};
\draw[black!23,line width=0.4pt] (10.94,1.85)--(13.70,1.85);
\node[gef label] at (11.81,1.64) {\(\widehat R_{(1)}\)};
\node[gef label] at (13.24,1.64) {\(n_1\)};
\node[gef label] at (11.81,1.21) {\(\widehat R_{(2)}\)};
\node[gef label] at (13.24,1.21) {\(n_2\)};
\node[gef label,text width=3.10cm] at (12.32,0.58)
  {decode, merge, rank by frequency};

\draw[gef flow,rounded corners=3pt]
  (8.99,1.45)--(8.99,0.42)--(2.08,0.42)--(2.08,1.70);
\fill[gefInk!78] (2.08,1.70) circle (0.025);
\node[gef label,fill=white] at (5.39,0.42)
  {refresh rates on the updated graph};
\end{tikzpicture}%
\endgroup%
}
\caption{\model{}: product-relative complete edits, bridge supervision, and generation.
Paired endpoints supervise edit rates (dashed arrows); the shared network
observes \((P,X_t,t)\). Each generated edit updates the graph before the
next rate prediction. States are schematic.}
\label{fig:graph-generator}
\end{figure}

\subsection{Dynamic graph states and complete edits}

Let \(P=(V_P,E_P)\) be the product graph and \(R=(V_R,E_R)\) the
disjoint union of reactant graphs, with atom and bond attributes including
stereochemistry. We learn \(p_\theta(R\mid P)\), suppressing any supplied
reaction-center condition throughout the notation. We describe \(R\) and
each evolving graph relative to the fixed product: added atoms, revised
attributes on inherited atoms, and added, removed, or relabelled bonds.
This description is lossless; the network receives the evolving graph's
changes relative to \(P\).
Atom mapping fixes the correspondence of inherited atoms to \(P\)
\citep{schwaller2021rxnmapper}. The state starts at \(X_0=P\), retaining
the product atoms and any generated atoms still present.

A complete action \(a\) specifies the location and attributes needed to
produce one successor graph \(T_aX\). An \emph{attached atom addition}
creates a new atom and its first bond; an \emph{isolated atom addition}
starts a new component. Other edits add or remove bonds, update attributes,
or remove generated atoms. Each update is completed before the next rate
prediction, so newly created structure can be extended or revised.

The admissible action set \(\mathcal A(X)\) preserves inherited product atoms and their
atomic numbers, excludes self-loops and duplicate bonds, and restricts
attribute choices to the supported vocabulary. Its masks are recomputed on
the current graph. Intermediate states may be incomplete chemical
structures; molecular validity is checked only after generation. The edit
vocabulary and structural masks are specified in
Appendices~\ref{app:events} and \ref{app:alg-factorization}.

\subsection{Bridges from products to reactants}

We construct a stochastic bridge from \(P\) to each recorded reactant graph
\(R\). The bridge creates missing atoms and corrects atoms or bonds that
differ from \(R\), producing intermediate graphs and local supervision
without a prescribed edit sequence. This use of endpoint-conditioned bridges
connects to bridge-based retrosynthesis and general Markov-generator matching
\citep{igashov2024retrobridge,holderrieth2024generatormatching}.

The sum of edit rates controls when the graph changes, while their normalized
values determine which edit occurs. We express both bridge and learned rates
per unit of transformed time \(\tau=-\log(1-t)\), so that
\(t\in[0,1)\) corresponds to \(\tau\in[0,\infty)\). The network receives
\(t\) as its bounded time input. Let
\(\mathcal C_R(X)\) denote the corrections other than atom addition,
\(m_R(X)\) the number of missing reactant-only atoms, and
\(\rho_R(a\mid X)\) a normalized distribution over eligible atom-addition
edits. The bridge assigns rate
\begin{equation}
 q_R(a\mid X)=
 \mathbf{1}[a\in\mathcal C_R(X)]+m_R(X)\rho_R(a\mid X),
 \label{eq:bridge-target}
\end{equation}
where the atom-addition term is zero when no atom is missing. Each other
correction has unit rate, while the total atom-addition rate equals the number
of atoms still to be created. This separates the amount of graph growth
remaining from the number of possible attachment choices. The total bridge intensity is
\(\Lambda_R(X)=\sum_a q_R(a\mid X)\). Every bridge edit reduces the discrepancy with
the target reactant graph. The bridge reaches a representative of the target
reactant graph after finitely many edits and remains there as \(t\to1\).

The bridge admits different edit orders, such as exchanging bond deletion
and atom addition (Figure~\ref{fig:edit-paths}A). A separate ambiguity is
which missing target atom a new atom realizes. For example, two target
atoms with the same attributes can propose the same attachment and first
bond. We sum the rates of identical complete updates, defining the
\emph{quotient bridge}. These hidden correspondences guide training paths,
while the network predicts observable edits. Appendix~\ref{app:events}
specifies the atom-addition rules.

\begin{figure}[tb!]
\centering
\begin{minipage}[t]{.48\linewidth}
\vspace{0pt}
\centering
\resizebox{\linewidth}{!}{
\begingroup%
\renewcommand{\sfdefault}{DejaVuSans-TLF}%
\definecolor{gepBlue}{HTML}{0077BB}%
\definecolor{gepOrange}{HTML}{EE7733}%
\definecolor{gepInk}{HTML}{27333C}%
\begin{tikzpicture}[
  x=1cm,y=1cm,
  font=\sffamily\fontsize{7}{8.4}\selectfont,
  text=gepInk,line cap=round,line join=round,
  gep anchor/.style={circle,draw=gepBlue,line width=0.8pt,
    fill=gepBlue!12,minimum size=3.6mm,inner sep=0pt},
  gep new/.style={circle,draw=gepOrange,line width=0.75pt,
    double=white,double distance=0.55pt,fill=gepOrange!13,
    minimum size=3.6mm,inner sep=0pt},
  gep bond/.style={draw=gepInk,line width=1.05pt},
  gep firstbond/.style={draw=gepOrange,line width=1.3pt},
  gep arrow/.style={-{Latex[length=1.65mm,width=1.2mm]},
    draw=gepInk!78,line width=0.8pt},
  gep label/.style={align=center,inner sep=1.2pt},
  gep heading/.style={gep label,
    font=\sffamily\bfseries\fontsize{7}{8.4}\selectfont},
  pics/gepproduct/.style={code={
    \draw[gep bond] (-0.16,0)--(0.41,0);
    \node[gep anchor] at (-0.16,0) {\(j\)};
    \node[gep anchor] at (0.41,0) {\(i\)};
  }},
  pics/gepdeleted/.style={code={
    \node[gep anchor] at (-0.16,0) {\(j\)};
    \node[gep anchor] at (0.41,0) {\(i\)};
  }},
  pics/gepgrown/.style={code={
    \draw[gep firstbond] (-0.73,0)--(-0.16,0);
    \draw[gep bond] (-0.16,0)--(0.41,0);
    \node[gep new] at (-0.73,0) {\(u\)};
    \node[gep anchor] at (-0.16,0) {\(j\)};
    \node[gep anchor] at (0.41,0) {\(i\)};
  }},
  pics/geptarget/.style={code={
    \draw[gep firstbond] (-0.73,0)--(-0.16,0);
    \node[gep new] at (-0.73,0) {\(u\)};
    \node[gep anchor] at (-0.16,0) {\(j\)};
    \node[gep anchor] at (0.41,0) {\(i\)};
  }}
]
\path[use as bounding box] (0,0) rectangle (6.72,3.0);
\node[gep heading,anchor=west] at (0.10,2.80)
  {A\quad Edit order};

\pic at (0.56,1.51) {gepproduct};
\node[gep heading] at (0.75,2.03) {product \(P\)};
\node[gep label] at (0.67,0.98) {\(\Lambda_R=2\)};

\pic at (3.28,2.33) {gepdeleted};
\node[gep label] at (3.36,2.80) {\(\Lambda_R=1\)};
\draw[gep arrow] (1.25,1.70)--(2.82,2.32);
\node[gep label] at (2.02,2.46) {remove bond};

\pic at (3.28,0.68) {gepgrown};
\node[gep label] at (3.06,0.23) {\(\Lambda_R=1\)};
\draw[gep arrow] (1.25,1.32)--(2.25,0.68);
\node[gep label] at (1.65,0.38) {add atom};

\pic at (5.84,1.51) {geptarget};
\node[gep heading] at (5.75,2.03) {same \(R\)};
\node[gep label] at (5.75,0.98) {\(\Lambda_R=0\)};
\node[gep label,text=black!65] at (5.75,0.57) {target reached};
\draw[gep arrow] (3.99,2.32)--(4.81,1.70);
\node[gep label] at (4.51,2.46) {add atom};
\draw[gep arrow] (3.99,0.68)--(4.81,1.32);
\node[gep label] at (4.65,0.18) {remove bond};

\node[gep label,text=gepOrange!85!black] at (3.31,1.51)
  {each arrow\\\(q_R(a\mid X)=1\)};
\end{tikzpicture}%
\endgroup%
}
\end{minipage}\hfill
\begin{minipage}[t]{.48\linewidth}
\vspace{0pt}
\centering
\resizebox{\linewidth}{!}{
\begingroup%
\renewcommand{\sfdefault}{DejaVuSans-TLF}%
\definecolor{payloadBlue}{HTML}{0077BB}%
\definecolor{payloadOrange}{HTML}{EE7733}%
\definecolor{payloadInk}{HTML}{27333C}%
\begin{tikzpicture}[x=1cm,y=1cm,
  font=\sffamily\fontsize{7}{8.4}\selectfont,
  text=payloadInk,
  lab/.style={inner sep=1pt,align=center},
  factor/.style={draw=black!30,rounded corners=2pt,fill=black!2,
    text width=1.75cm,minimum height=.90cm,align=center,inner sep=2pt},
  atom/.style={circle,draw=payloadBlue,fill=payloadBlue!10,
    minimum size=4mm,inner sep=0pt,line width=.7pt},
  new/.style={atom,draw=payloadOrange,fill=payloadOrange!12,
    double=white,double distance=.4pt},
  flow/.style={-{Latex[length=1.5mm]},line width=.7pt,draw=black!65}]
\path[use as bounding box] (0,0) rectangle (6.72,3.0);
\node[lab,anchor=west,font=\sffamily\bfseries\fontsize{7}{8.4}\selectfont]
  at (.02,2.83) {B\quad One atom-addition edit};
\node[factor] (location) at (1.0,2.02)
  {attach to \(j\)\\element \(z\)};
\node[factor] (mark) at (3.3,2.02)
  {\(\mathbf m\)};
\node[factor] (bond) at (5.6,2.02)
  {first bond\\\(\mathbf e\mid\mathbf m,j\)};
\draw[flow] (location)--(mark);
\draw[flow] (mark)--(bond);
\node[atom] (j) at (.45,.65) {\(j\)};
\node[atom] (i) at (1.1,.65) {\(i\)};
\draw[line width=.9pt] (i)--(j);
\node[atom] (jj) at (5.20,.65) {\(j\)};
\node[atom] (ii) at (5.85,.65) {\(i\)};
\node[new] (u) at (5.20,1.23) {\(u\)};
\draw[line width=.9pt] (ii)--(jj);
\draw[payloadOrange,line width=1pt] (jj)--(u);
\draw[flow,line width=.9pt] (1.65,.65)--(4.60,.65);
\node[lab] at (3.1,1.00) {\((u,j,\mathbf m,\mathbf e)\)};
\node[lab,text=black!65] at (3.1,.27) {one graph update};
\node[lab] at (.78,.17) {\(X\)};
\node[lab] at (5.55,.17) {\(T_aX\)};
\end{tikzpicture}%
\endgroup%
}
\end{minipage}
\caption{Edit order and complete graph updates.
A: Two correction orders reach the same graph, with total bridge
intensity \(2\to1\to0\).
B: Location and attributes jointly define one update.
Blue: product atoms; orange: added atoms.}
\label{fig:edit-paths}
\label{fig:complete-edit}
\end{figure}
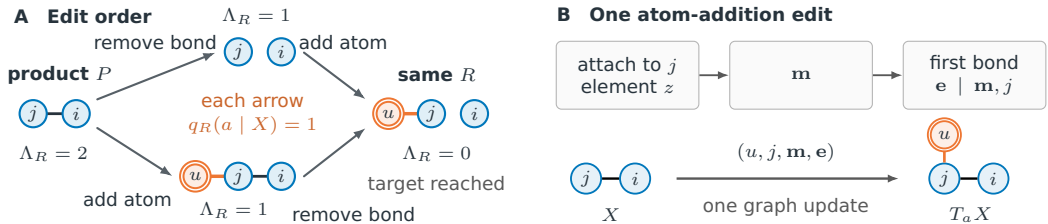

\subsection{Parameterizing graph edit rates}

The rate of a complete edit factors as
\begin{equation}
 r_\theta(a\mid X,P,t)=
 \lambda_\theta(X,P,t)\,\pi_\theta(a\mid X,P,t),
 \label{eq:rate-factorization}
\end{equation}
where \(\lambda_\theta>0\) is the total edit intensity (the sum of all edit rates) and
\(\sum_{a\in\mathcal A(X)}\pi_\theta(a\mid X,P,t)=1\).
The intensity controls the waiting time until the next edit, and
\(\pi_\theta\) determines which edit occurs.

A graph Transformer \citep{vaswani2017attention} encodes the current graph,
the time, and atom and bond changes relative to the product. Bond features
inform attention, while graph-level features summarize the amount of
structural change. A graph readout predicts total intensity, and
conditional distributions specify the edit location and attributes
(Figure~\ref{fig:model-architecture}). Architecture and training settings
appear in Appendix~\ref{app:model-training}.

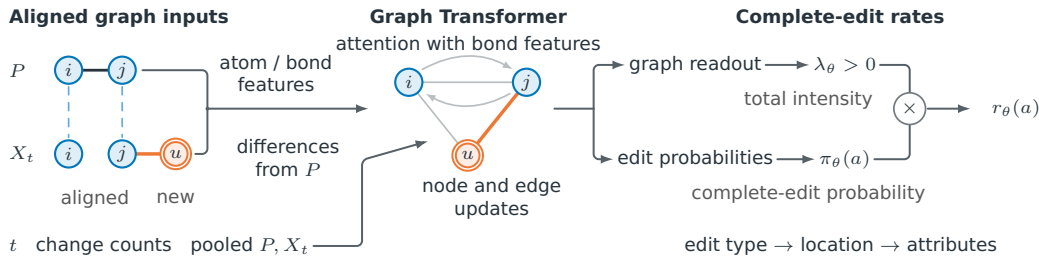
\begin{figure}[tb!]
\centering
\resizebox{\textwidth}{!}{
\begingroup%
\renewcommand{\sfdefault}{DejaVuSans-TLF}%
\definecolor{gefArchBlue}{HTML}{0077BB}%
\definecolor{gefArchOrange}{HTML}{EE7733}%
\definecolor{gefArchInk}{HTML}{27333C}%
\begin{tikzpicture}[
  x=1cm,y=1cm,
  font=\sffamily\fontsize{7}{8.4}\selectfont,
  text=gefArchInk,line cap=round,line join=round,
  arch label/.style={align=center,inner sep=1.3pt},
  arch title/.style={arch label,
    font=\sffamily\bfseries\fontsize{7}{8.4}\selectfont},
  arch atom/.style={circle,draw=gefArchBlue,line width=0.8pt,
    fill=gefArchBlue!12,minimum size=3.6mm,inner sep=0pt},
  arch new/.style={circle,draw=gefArchOrange,line width=0.75pt,
    double=white,double distance=0.55pt,fill=gefArchOrange!13,
    minimum size=3.6mm,inner sep=0pt},
  arch bond/.style={draw=gefArchInk,line width=1.05pt},
  arch wire/.style={draw=gefArchInk!78,line width=0.75pt},
  arch flow/.style={arch wire,-{Latex[length=1.7mm,width=1.2mm]}},
  arch pair/.style={draw=gefArchInk!32,line width=0.65pt},
  arch match/.style={draw=gefArchBlue!60,dashed,line width=0.65pt}
]
\path[use as bounding box] (0,0) rectangle (14,3.55);

\node[arch title,anchor=west] at (0.03,3.30) {Aligned graph inputs};
\node[arch title] at (6.17,3.30) {Graph Transformer};
\node[arch title] at (11.10,3.30) {Complete-edit rates};

\node[arch label,anchor=west] at (0.03,2.61) {\(P\)};
\node[arch label,anchor=west] at (0.03,1.50) {\(X_t\)};
\draw[arch bond] (0.87,2.61)--(1.58,2.61);
\draw[arch match] (0.87,2.39)--(0.87,1.72);
\draw[arch match] (1.58,2.39)--(1.58,1.72);
\draw[draw=gefArchOrange,line width=1.3pt] (1.58,1.50)--(2.29,1.50);
\node[arch atom] at (0.87,2.61) {\(i\)};
\node[arch atom] at (1.58,2.61) {\(j\)};
\node[arch atom] at (0.87,1.50) {\(i\)};
\node[arch atom] at (1.58,1.50) {\(j\)};
\node[arch new] at (2.29,1.50) {\(u\)};
\node[arch label,text=black!65] at (1.22,0.91) {aligned};
\node[arch label,text=black!65] at (2.29,0.91) {new};

\draw[arch wire,rounded corners=2pt]
  (1.85,2.61)--(2.71,2.61)--(2.71,2.10);
\draw[arch wire,rounded corners=2pt]
  (2.55,1.50)--(2.71,1.50)--(2.71,2.10);
\fill[gefArchInk!78] (2.71,2.10) circle (0.025);
\draw[arch flow] (2.71,2.10)--(4.85,2.10);
\node[arch label] at (3.62,2.57) {atom / bond\\features};
\node[arch label] at (3.78,1.45) {differences\\from \(P\)};

\draw[arch pair] (5.38,2.45)--(6.94,2.45);
\draw[arch pair] (5.38,2.45)--(6.16,1.48);
\draw[draw=gefArchOrange,line width=1.3pt] (6.94,2.45)--(6.16,1.48);
\draw[arch pair,-{Latex[length=1.35mm,width=1.0mm]}]
  (5.56,2.62) to[out=29,in=151] (6.76,2.62);
\draw[arch pair,-{Latex[length=1.35mm,width=1.0mm]}]
  (6.72,2.30) to[out=209,in=-29] (5.60,2.30);
\node[arch atom] at (5.38,2.45) {\(i\)};
\node[arch atom] at (6.94,2.45) {\(j\)};
\node[arch new] at (6.16,1.48) {\(u\)};
\node[arch label] at (6.17,2.97) {attention with bond features};
\node[arch label] at (6.48,0.91) {node and edge\\updates};

\node[arch label,anchor=west] (archContext) at (0.03,0.28)
  {\(t\)\quad change counts\quad pooled \(P,X_t\)};
\draw[arch flow,rounded corners=2pt]
  (archContext.east)--(4.76,0.28)--(4.76,1.42)--(5.64,1.67);

\draw[arch wire] (7.31,2.10)--(7.83,2.10);
\node[arch label] (archPool) at (9.18,2.68) {graph readout};
\node[arch label] (archHeads) at (9.18,1.43) {edit probabilities};
\node[arch label] (archLambda) at (11.16,2.68) {\(\lambda_\theta>0\)};
\node[arch label] (archPi) at (11.16,1.43) {\(\pi_\theta(a)\)};
\draw[arch flow,rounded corners=2pt]
  (7.83,2.10)--(7.83,2.68)--(archPool.west);
\draw[arch flow,rounded corners=2pt]
  (7.83,2.10)--(7.83,1.43)--(archHeads.west);
\draw[arch flow] (archPool.east)--(archLambda.west);
\draw[arch flow] (archHeads.east)--(archPi.west);
\node[arch label,text=black!65] at (10.68,2.22) {total intensity};
\node[arch label,text=black!65] at (10.68,0.97) {complete-edit probability};

\draw[arch wire,rounded corners=2pt]
  (archLambda.east)--(12.02,2.68)--(12.02,2.10);
\draw[arch wire,rounded corners=2pt]
  (archPi.east)--(12.02,1.43)--(12.02,2.10);
\node[circle,fill=white,draw=gefArchInk!65,line width=0.65pt,
  minimum size=4.2mm,inner sep=0pt] at (12.02,2.10) {\(\times\)};
\draw[arch flow] (12.26,2.10)--(12.83,2.10);
\node[arch label] at (13.43,2.10) {\(r_\theta(a)\)};
\node[arch label] at (11.11,0.28)
  {edit type \(\to\) location \(\to\) attributes};
\end{tikzpicture}%
\endgroup%
}
\caption{Edit-rate prediction. Intensity and complete-edit probability
determine each rate. Dashed links align product atoms; gray links inside the
Transformer denote feature interactions, not molecular bonds.}
\label{fig:model-architecture}
\end{figure}

\paragraph{Edit probabilities.}
The action distribution accounts for both the edit location and its attributes. To
illustrate, let \(Z=(P,X,t)\), let \(\mathbf m=(z,m_2,\ldots,m_6)\)
be a tuple of atom attributes with atomic number \(z\), and let \(\mathbf e\)
be a tuple of bond attributes. For an attached atom addition, denoted \(f_{\mathrm A}\), at existing atom \(j\),
the probability factors as
\begin{align}
 \pi_\theta(f_{\mathrm A},j,\mathbf m,\mathbf e\mid Z)
 &=\pi_F(f_{\mathrm A}\mid Z)\,
   \pi_{J,z}(j,z\mid f_{\mathrm A},Z)\nonumber\\
 &\quad\times\pi_M(m_{2:6}\mid j,z,Z)\,
   \pi_E(\mathbf e\mid\mathbf m,j,Z).
 \label{eq:complete-action-factorization}
\end{align}
The attachment site and new element are selected jointly, followed by
autoregressive prediction of the remaining atom attributes. Conditioned
on the complete new atom, the bond decoder jointly predicts the first
bond's type and stereochemistry.
These factors define one update (Figure~\ref{fig:complete-edit}B); the
remaining edit types use analogous factorizations. Training and sampling
use the same complete-action probabilities
(Appendix~\ref{app:alg-factorization}).

\subsection{Learning from endpoint bridges}

For a sampled training time \(t\), we simulate the bridge to obtain
\(X_t\) and its target rates \(q_R(a\mid X_t)\). We minimize generalized KL
divergence between target and predicted rates. Omitting a target-only
constant, the loss for one training state is
\begin{equation}
 \mathcal L_{\mathrm{rate}}=
 \lambda_\theta-\Lambda_R\log\lambda_\theta
 -\sum_{a:q_R(a\mid X_t)>0}q_R(a\mid X_t)
    \log\pi_\theta(a\mid X_t,P,t).
 \label{eq:gkl}
\end{equation}
The first two terms fit total intensity, and the last fits the corrective-edit
distribution. Only actions with nonzero target rates require explicit
probabilities: normalization of \(\pi_\theta\) makes the summed predicted
rate over all admissible actions equal to \(\lambda_\theta\).

Unlike normalized next-edit classification, rate matching also fits the
pace of graph change. At an endpoint, \(\Lambda_R=0\) and the loss reduces
to \(\lambda_\theta\), supervising low outgoing intensity without a stop
label. Appendix~\ref{app:rate-decomposition} derives the decomposition into intensity
and action-choice terms, including this zero-rate case.

For unperturbed bridges, the unrestricted population optimum over
nonnegative rates is
\begin{equation}
 r^*(a\mid X,P,t)=\mathbb E[q_R(a\mid X)\mid X_t=X,P,t].
 \label{eq:main-marginal-rate}
\end{equation}
The expectation includes hidden endpoints and atom correspondences.
Under the conditions of Proposition~\ref{prop:main-marginal} in
Appendix~\ref{app:marginal-consistency}, these rates define a product-initialized
process with the conditional bridge's one-time graph marginals
\citep{holderrieth2024generatormatching}.
Since each bridge reaches its target, this idealized process recovers the
conditional reactant distribution of the paired data as \(t\to1\).

\paragraph{Recovery augmentation.}
Unperturbed bridges add only target atoms. To supervise removal and repair,
we perturb bridge states with unwanted bonds or superfluous generated leaf
atoms and recompute corrective rates toward the same endpoint. This trains
deletion and structural revision alongside graph growth, without a prescribed
recovery sequence (Appendix~\ref{app:model-training}).

\subsection{Reactant generation}

At inference, \model{} starts from the product and samples edit times and
actions from the learned rates. The time grid controls numerical integration,
not the number of edits. Rates are recomputed after every edit, so several
edits can occur within one interval and newly added atoms are immediately
available for attachment or revision. At a finite terminal time, graphs are
decoded into reactant sets. Because every action is applied to the current
graph, the same representation can accommodate different numbers of
components and edits without prescribing their order. Growth, disconnection,
and scaffold revision consequently remain within one generation process.
A newly introduced component can be extended by subsequent edits, and its
bonds or attributes can be revised as the surrounding graph changes.
Likewise, inherited product structure remains available as context for later
actions rather than being frozen after the initial prediction. The terminal
graph therefore records missing component structure and changes to the
retained scaffold in one state description. This shared process does not
require a separate stage for each transformation size or a predefined order
of operations. The terminal distribution is consequently defined by the same
edit space that is supervised during training.
Appendix~\ref{app:alg-sampling} gives the sampling details.

\section{Experiments}
\label{sec:experiments}

\subsection{Experimental setup}

\captionsetup{font=small,position=top,skip=3pt,singlelinecheck=false}
\vspace{-8pt}
\begin{wraptable}[27]{r}{0.50\textwidth}
\vspace{-\intextsep}
\centering
\caption{USPTO-Full test results. Top-$k$ accuracy (\%).}
\label{tab:full-results}
\tableformat
\renewcommand{\arraystretch}{0.98}
\begin{tabular}{@{}lrrrr@{}}
\toprule
Method & Top-1 & Top-3 & Top-5 & Top-10 \\
\midrule
RetroSim & 32.8 & -- & -- & 56.1 \\
LocalRetro & 39.1 & 53.3 & 58.4 & 63.7 \\
GLN & 39.3 & -- & -- & 63.7 \\
RetroPrime & 44.1 & 59.1 & 62.8 & 68.5 \\
R-SMILES & 48.9 & 66.6 & 72.0 & 76.4 \\
NAG2G & 49.7 & 64.6 & 69.3 & 74.0 \\
\midrule
RetroDiT (Pred.) & 51.3 & 67.8 & 72.3 & 75.8 \\
\textbf{\model{} (Pred.)} & \textbf{54.3} & \textbf{70.1} & \textbf{74.1} & \textbf{77.1} \\
\midrule
RSGPT & 59.2 & 74.2 & 78.2 & 82.1 \\
RetroDiT (Oracle) & 63.4 & 77.6 & 80.9 & 83.6 \\
\textbf{\model{} (Oracle)} & \textbf{68.6} & \textbf{80.7} & \textbf{82.8} & \textbf{84.1} \\
\bottomrule
\end{tabular}
\vspace{1pt}
\caption{USPTO-50K test results. Top-$k$ accuracy (\%).}
\label{tab:50k-results}
\tableformat
\begin{tabular}{@{}lrrrr@{}}
\toprule
Method & Top-1 & Top-3 & Top-5 & Top-10 \\
\midrule
RetroSim & 37.3 & 54.7 & 63.6 & 74.1 \\
Retro ProdFlow~\citep{yadav2025retrosynflow} & 50.0 & 74.3 & 81.2 & 85.8 \\
RetroBridge & 50.8 & 74.1 & 80.6 & 85.6 \\
RetroPrime & 51.4 & 70.8 & 74.0 & 76.1 \\
GLN & 52.5 & 69.0 & 75.6 & 83.7 \\
LocalRetro & 53.4 & 77.5 & 85.9 & 92.4 \\
NAG2G & 55.1 & 76.9 & 83.4 & 89.9 \\
Graph2Edits & 55.1 & 77.3 & 83.4 & 89.4 \\
R-SMILES & 56.3 & 79.2 & 86.2 & 91.0 \\
Retro SynFlow & 60.0 & 77.9 & 82.7 & 85.3 \\
\midrule
RetroDiT (Pred.) & 61.2 & 81.5 & 86.2 & 89.2 \\
\textbf{\model{} (Pred.)} & \textbf{62.4} & \textbf{82.8} & \textbf{88.6} & \textbf{91.9} \\
\midrule
RSGPT$^\ast$ & 26.4 & 37.5 & 41.4 & 46.4 \\
RSGPT & 63.4 & 84.2 & 89.2 & 93.0 \\
RetroDiT (Oracle) & 71.1 & 90.8 & 94.5 & 96.0 \\
\textbf{\model{} (Oracle)} & \textbf{72.3} & \textbf{92.9} & \textbf{96.0} & \textbf{97.0} \\
\bottomrule
\end{tabular}
\vspace{-2pt}
\end{wraptable}

We evaluate on USPTO-Full and USPTO-50K, derived from patent reaction
records \citep{schneider2016roles,lowe2017uspto}, using the benchmark splits
distributed with RetroDiT \citep{wang2026retrodit}.
Models are trained separately on each dataset and evaluated with predicted
(Pred.) or reference (Oracle) reaction-center information. We sample 100
trajectories per product using 50 sampling intervals, merge equivalent
reactant sets, and rank them by frequency. Top-$k$ accuracy follows the
atom-and-bond topology matching of \citet{wang2026retrodit}.
Appendix~\ref{app:implementation} specifies the architecture, training,
decoding configurations, and matching rules.

The comparisons cover complementary approaches to retrosynthesis:
template-based prediction with RetroSim~\citep{coley2017retrosim},
GLN~\citep{dai2019gln}, and LocalRetro~\citep{chen2021localretro};
sequence generation with RetroPrime~\citep{wang2021retroprime},
R-SMILES~\citep{zhong2022rsmiles}, and RSGPT~\citep{deng2025rsgpt};
graph translation and editing with NAG2G~\citep{yao2023nag2g} and
Graph2Edits~\citep{zhong2023graph2edits}; and stochastic graph generation
with RetroBridge~\citep{igashov2024retrobridge}, Retro ProdFlow and
Retro SynFlow~\citep{yadav2025retrosynflow}, and
RetroDiT~\citep{wang2026retrodit}. We report the baseline scores from the cited
sources,
with Pred. and Oracle comparisons shown separately for center-conditioned
generation. RSGPT includes its pretrained model and fine-tuning-only comparison
($^\ast$). Appendix~\ref{app:benchmark-conventions} gives the benchmark
conventions.

\subsection{Main results}

With predicted centers, \model{} reaches Top-1/Top-10 accuracies of
54.3/77.1\% on USPTO-Full and 62.4/91.9\% on USPTO-50K
(Tables~\ref{tab:full-results} and~\ref{tab:50k-results}).
Reference centers further improve these results to 68.6/84.1\% and
72.3/97.0\%, respectively. Within the center-conditioned comparisons,
\model{} exceeds RetroDiT at all four cutoffs on both datasets.
The USPTO-Full Top-1 gains are 3.0 points under Pred. and 5.2 under Oracle;
on USPTO-50K, the corresponding gains are 1.2 points in both conditions,
with Top-10 increasing by 2.7 and 1.0 points. The gains at both ends of the
ranking indicate more accurate first predictions as well as better
recovery of the recorded reactants among the ranked alternatives.
Top-1 therefore measures the reliability of the first proposed reactant set,
while Top-10 tests whether the ranking preserves complementary valid sets.
The gains at both cutoffs show that the improvement is not only broader
coverage: it also changes which candidate is ranked first.
In both settings, the center signal identifies where a transformation may
occur but leaves the missing reactant structure to be generated. Recovering
complete candidates therefore still requires decisions about fragment size,
connectivity, and revisions to the product graph. The two benchmarks show
that learning these decisions through complete edits can support accurate
ranking across distinct levels of reaction-center information.
The ordering is consistent across the two datasets despite their different
scales and product distributions: USPTO-Full stresses variable-size
components, whereas USPTO-50K provides a more constrained class composition.
The Pred.--Oracle gap therefore reflects the information supplied about the
reaction center, while the decoder still has to choose component structure and
connectivity in both cases. This separation motivates the structural analyses
below.

\subsection{Graph growth and transformation complexity}
\label{sec:graph-growth}
\label{sec:structural-difficulty}

Reactant reconstruction requires both introducing atoms absent from the
product and revising inherited structure. We examine these demands
separately on 90,598 shared USPTO-Full products, using paired predictions
from \model{} and RetroDiT \citep{wang2026retrodit}.
Grouping by added atoms (Table~\ref{tab:full-growth-main}), \model{}
retains a Top-1 advantage in every group requiring graph growth under
both center conditions. For 11--20 added atoms, Oracle Top-1 rises from
26.0\% to 33.4\% and Top-3 from 40.1\% to 45.0\%.
Among these products, 823 have a correct first prediction only from
\model{}, compared with 406 only from RetroDiT, giving 417 additional
correct first predictions. At the other end of the range, reactions
requiring no added atoms isolate reconstruction through revisions to
existing structure: Oracle Top-1 increases from 71.9\% to 84.8\%, with
gains at all four cutoffs in both conditions. The advantage therefore
extends to both growth-intensive reactions and transformations confined
to the inherited atoms.

\begin{wraptable}[9]{r}{0.49\textwidth}
\vspace{-\intextsep}
\centering
\captionsetup{font=small,skip=3pt,justification=raggedright,singlelinecheck=false}
\caption{Transformation complexity on USPTO-Full. Top-1 accuracy (\%) on
90,598 shared test products.}
\label{tab:transformation-complexity-main}
\tableformat
\begin{tabular}{@{}lrrrr@{}}
\toprule
 & \multicolumn{2}{c}{Pred.} & \multicolumn{2}{c}{Oracle} \\
\cmidrule(lr){2-3}\cmidrule(lr){4-5}
Graph changes & RetroDiT & \model{} & RetroDiT & \model{} \\
\midrule
$\leq5$ & 65.1 & \textbf{67.6} & 78.7 & \textbf{81.9} \\
6--10 & 45.1 & \textbf{48.8} & 57.3 & \textbf{63.2} \\
11--20 & 32.6 & \textbf{36.8} & 43.8 & \textbf{52.1} \\
$>20$ & 25.7 & \textbf{29.5} & 33.6 & \textbf{42.5} \\
\bottomrule
\end{tabular}
\vspace{-2pt}
\end{wraptable}

The number of added atoms alone does not describe the extent of a
transformation. Table~\ref{tab:transformation-complexity-main} instead
groups reactions by atom additions, updated product-atom attributes, and
added, removed, or modified bonds, without assuming an edit order.
\model{} improves Top-1 in every group under both conditions, and the
Oracle advantage widens as the number of changes increases.
For more than twenty changes, Oracle Top-1 reaches
42.5\% versus 33.6\%, and Pred. reaches 29.5\% versus 25.7\%, with gains
at all four cutoffs. These results show that the improvement persists
when reconstructing the reactants involves many structural changes,
not just the introduction of additional atoms. The complete paths in
Appendix~\ref{app:sampled-trajectories} illustrate the same behavior at the
trajectory level, including examples that alternate growth with bond and
attribute updates.

\begin{wraptable}[10]{r}{0.52\textwidth}
\vspace{-\intextsep}
\centering
\captionsetup{font=small,skip=3pt,justification=raggedright,singlelinecheck=false}
\caption{USPTO-Full by added atoms. Oracle test accuracy (\%) under the
common decoding budget.}
\label{tab:full-growth-main}
\tableformat
\begin{tabular}{@{}l@{\hspace{2pt}}rrrr@{\hspace{2pt}}rrrr@{}}
\toprule
\multirow{2}{*}{\shortstack[l]{Added\\atoms}}
 & \multicolumn{4}{c}{\model{}} & \multicolumn{4}{c}{RetroDiT} \\
\cmidrule(lr){2-5}\cmidrule(lr){6-9}
 & 1 & 3 & 5 & 10 & 1 & 3 & 5 & 10 \\
\midrule
0 & \textbf{84.8} & \textbf{91.5} & \textbf{93.0} & \textbf{93.6}
 & 71.9 & 85.4 & 88.2 & 90.5 \\
1--2 & \textbf{79.2} & \textbf{90.6} & \textbf{92.0} & \textbf{92.7}
 & 76.4 & 88.6 & 90.8 & 92.4 \\
3--5 & \textbf{56.0} & \textbf{71.1} & \textbf{74.7} & \textbf{77.0}
 & 49.4 & 67.7 & 72.5 & 76.6 \\
6--10 & \textbf{54.4} & \textbf{68.1} & \textbf{70.9} & \textbf{73.1}
 & 47.1 & 63.8 & 68.6 & 72.4 \\
11--20 & \textbf{33.4} & \textbf{45.0} & \textbf{47.8} & 50.0
 & 26.0 & 40.1 & 45.6 & \textbf{50.4} \\
\bottomrule
\end{tabular}
\vspace{-2pt}
\end{wraptable}

\begin{figure}[tb!]
\centering
\includegraphics[width=\linewidth]{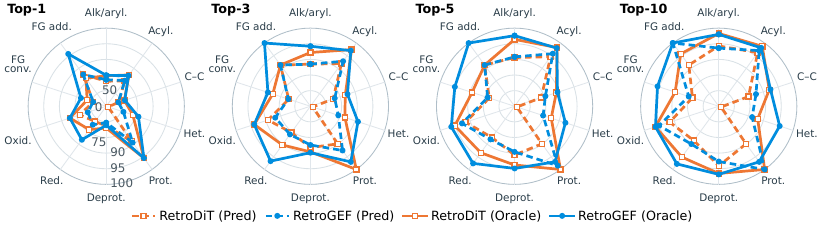}
\caption{Accuracy across all ten USPTO-50K reaction classes.
Shared nonuniform radial scale (\%), labeled in the first panel.
FG: functional group.}
\label{fig:reaction-class-profiles}
\end{figure}

\subsection{Performance across reaction classes}
\label{sec:reaction-classes}

Reaction classes provide a chemical perspective complementary to the
structural groupings above. Figure~\ref{fig:reaction-class-profiles}
compares all ten USPTO-50K classes \citep{schneider2016roles} on 4,930
shared products, using the same edit vocabulary and generation process
throughout. The improvements cover several distinct transformation
families. Reduction improves at every cutoff in both conditions: Pred.
Top-1/Top-10 rises from 55.2/85.7\% to 61.6/89.8\%, while Oracle rises
from 71.3/95.1\% to 84.8/97.8\%. Heterocycle formation also improves
throughout the Oracle ranking, reaching 77.5\% Top-1 and 95.5\% Top-10
versus 69.7\% and 84.3\%. Under Pred., C--C formation improves Top-3/5/10
from 63.2/69.9/75.0\% to 65.9/76.1/82.6\%, and functional-group
interconversion improves Top-10 from 69.2\% to 75.3\%; its Oracle
Top-10 rises from 89.0\% to 95.6\%. These results place the gains across
reduction, ring formation, carbon--carbon coupling, and functional-group
changes, rather than restricting them to one reaction family.
Table~\ref{tab:reaction-classes} provides the complete class-wise results.

Taken together, the structural and class-wise analyses separate two aspects
of the task. Added-atom counts measure how much new material must enter the
reactant set, whereas graph-change counts also include revisions to inherited
atoms and bonds. The gains in both views, including the zero-added-atom group,
show that the method is not relying on growth alone. A single complete-edit
state can instead alternate between adding a component, changing its
connectivity, and revising the product-derived scaffold. The complete
class-wise values behind Figure~\ref{fig:reaction-class-profiles} are given in
Appendix~\ref{app:paired-groups}, allowing the same pattern to be inspected
across chemically distinct reaction families.

\Needspace{17\baselineskip}
\subsection{Training and sampling ablations}
\label{sec:ablations}

\begin{wraptable}{r}{0.42\textwidth}
\vspace{-\intextsep}
\centering
\captionsetup{font=small,skip=3pt,justification=raggedright,singlelinecheck=false}
\caption{Training and architecture variants. Test accuracy (\%) on
USPTO-50K and USPTO-Full.}
\label{tab:architecture-ablation-main}
\tableformat
\begin{tabular*}{\linewidth}{@{\extracolsep{\fill}}lrr@{\hspace{6pt}}rr@{}}
\toprule
 & \multicolumn{2}{c}{Oracle} & \multicolumn{2}{c}{Pred.} \\
\cmidrule(lr){2-3}\cmidrule(lr){4-5}
Variant & Top-1 & Top-10 & Top-1 & Top-10 \\
\midrule
\multicolumn{5}{@{}l}{\emph{USPTO-50K}} \\
Reference & 72.3 & \textbf{97.0} & \textbf{62.4} & \textbf{91.9} \\
Constant addition rate & 72.1 & 93.0 & 61.4 & 88.1 \\
Uniform time sampling & 71.6 & 95.8 & 61.2 & 90.5 \\
No edge-feature updates & \textbf{73.8} & 96.4 & \textbf{62.4} & 91.5 \\
\midrule
\multicolumn{5}{@{}l}{\emph{USPTO-Full}} \\
Reference & \textbf{68.6} & 84.1 & \textbf{54.3} & \textbf{77.1} \\
Constant addition rate & 64.9 & 79.3 & 50.0 & 70.1 \\
No edge-feature updates & 68.4 & \textbf{84.7} & 52.1 & 74.8 \\
\bottomrule
\end{tabular*}
\vspace{5pt}
\caption{Supervision comparison on USPTO-50K. Accuracy (\%): reference on test,
training variants on validation. Column maxima are bold.}
\label{tab:supervision-main}
\tableformat
\begin{tabular*}{\linewidth}{@{\extracolsep{\fill}}lrr@{\hspace{6pt}}rr@{}}
\toprule
 & \multicolumn{2}{c}{Pred.} & \multicolumn{2}{c}{Oracle} \\
\cmidrule(lr){2-3}\cmidrule(lr){4-5}
Training & Top-1 & Top-10 & Top-1 & Top-10 \\
\midrule
Reference & 62.4 & \textbf{91.9} & 72.3 & \textbf{97.0} \\
\midrule
+ Partial hints & 62.8 & 90.4 & 72.9 & 94.1 \\
+ Recovery & 60.2 & 88.0 & \textbf{73.8} & 95.6 \\
+ Both & \textbf{63.3} & 90.2 & 73.2 & 95.1 \\
\bottomrule
\end{tabular*}
\vspace{5pt}
\caption{Sampler ablation. USPTO-50K Oracle test accuracy (\%).}
\label{tab:sampling-main}
\tableformat
\begin{tabular*}{\linewidth}{@{\extracolsep{\fill}}lrrrr@{}}
\toprule
Sampling & Top-1 & Top-3 & Top-5 & Top-10 \\
\midrule
Stratified & 71.2 & 91.5 & 94.7 & 95.6 \\
\textbf{Independent} & \textbf{72.3} & \textbf{92.9} & \textbf{96.0} & \textbf{97.0} \\
\bottomrule
\end{tabular*}
\vspace{-2pt}
\end{wraptable}

Table~\ref{tab:architecture-ablation-main} compares training and architecture
variants with the reference results. Constant addition rate fixes the total growth
rate to one while atoms remain missing, whereas uniform time sampling removes
the late-state emphasis in the training distribution. The edge-feature
variant removes Transformer edge-state updates while retaining bond inputs and
bond-edit actions. The constant-rate variant reaches Top-10 of 79.3\% under
Oracle and 70.1\% under Pred. on USPTO-Full, below the reference in both cases.
The USPTO-50K variants likewise show lower Top-10 scores with constant-rate
addition or uniform time sampling. Edge-feature updates have a smaller and
less consistent effect across the two conditions. Thus, the rate controls
change how the model allocates capacity over the construction path, whereas
the edge-state control leaves the principal growth signal intact.
The complete rankings are given in Appendix~\ref{app:architecture-variants}.

Complete hints supply the full reference center set; partial hints replace it
with one reference-center atom on half the training examples. Recovery
perturbs intermediate graphs while retaining the original reactant target.
The training variants in Table~\ref{tab:supervision-main} use a fixed training
budget and one shared center predictor; validation uses 100 trajectories and 50
sampling intervals. Partial hints reach Pred. Top-1/Top-10 of 62.8/90.4\%,
while recovery reaches Oracle Top-10 of 95.6\%.
Combining both changes gives the highest Pred. Top-1 among these validation
variants, 63.3\%. The two training choices emphasize different parts of the
conditional construction problem: partial hints improve localization under
Pred., while recovery improves completion from a reference center.
The complete rankings are provided in Appendix~\ref{app:training-controls}.

Table~\ref{tab:sampling-main} compares independent and stratified sampling,
each using 100 trajectories and 50 sampling intervals.
Independent sampling draws random numbers separately for each trajectory;
stratified sampling partitions \([0,1]\) into 100 equal intervals and assigns
one to each trajectory for each edit-selection draw.
Both procedures keep the model and candidate budget unchanged, so this
comparison concerns the allocation of sampling variability rather than model
capacity. Independent sampling remains the stronger configuration at the reported
cutoffs. We therefore use it for the main comparisons and retain the
stratified result as a controlled sampling check.

\Needspace{14\baselineskip}
\subsection{Reactant construction examples}

\begin{wrapfigure}[16]{r}{0.43\textwidth}
\vspace{-\intextsep}
\centering
\includegraphics[width=\linewidth]{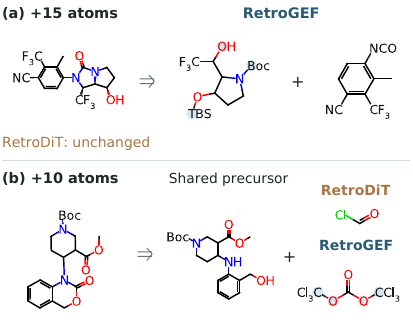}
\captionsetup{skip=3pt,justification=raggedright}
\caption{USPTO-Full Oracle Top-1 predictions.}
\label{fig:fragment-comparison}
\end{wrapfigure}

Complete reactant reconstruction can require recovering separate components
alongside changes to the product-derived scaffold.
Figure~\ref{fig:fragment-comparison} shows two Oracle examples involving
15 and 10 additional atoms.
In (a), \model{} recovers the Boc- and silyl-protected precursor together
with an aryl isocyanate, whereas RetroDiT leaves the product unchanged.
In (b), both methods recover the ring-opened amino alcohol, but only
\model{} reconstructs triphosgene; RetroDiT predicts $\mathrm{O{=}CCl}$.
Both \model{} predictions match the complete reference topology.
The second case separates scaffold recovery from reagent completion:
the correct scaffold alone is insufficient. A growing graph represents
both within one state, allowing generated structure to receive further
edits alongside revisions to inherited atoms.

\begin{figure}[tb!]
\vspace{-3pt}
\centering
\includegraphics[width=\linewidth]{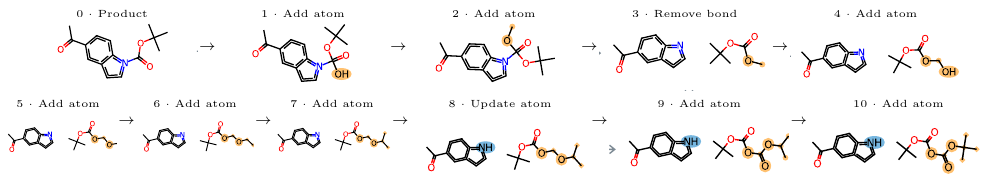}
\captionsetup{font=small,skip=2pt,justification=centering,singlelinecheck=false}
\caption{A sampled ten-edit construction path for a USPTO-50K protection
reaction. The two rows list edits 0--4 and 5--10; the final panel is the
complete reactant set.}
\label{fig:protection-path}
\end{figure}

Figure~\ref{fig:protection-path} follows this construction through ten
edits, from a 19-atom protected indole to a two-component reactant set
with eight additional atoms. Two additions precede cleavage of the
N--carbonyl connection at edit three. Growth then continues on the
separated fragment, completing di-tert-butyl dicarbonate, while an atom
update restores the indole nitrogen. Disconnection therefore need not finish
before growth begins: generated structure remains available for revision
throughout the path. This is a graph-construction trace, not a chemical
mechanism; the paired record specifies reactants but not this order.
Appendix~\ref{app:sampled-trajectories} provides paths from all ten
USPTO-50K reaction classes and three USPTO-Full graph-edit families,
including a 14-edit construction from a 48-atom product to a 52-atom
reactant set.

The path makes the variable-size representation concrete: new atoms enter only
when requested, and the resulting component remains available for later edits.
Pred. and Oracle change the initial center information, but both construct the
same state; the gains across added-atom, graph-change, and reaction-class groups
therefore reflect one shared edit space. The endpoint formulation supplies
intermediate states without a recorded order, so the sampled path shows one
consistent construction order rather than a chemical mechanism, alternating
between component growth, connectivity change, and inherited-atom revision.

\section{Conclusion}
\label{sec:conclusion}

\model{} learns variable-size reactant construction through continuous-time
complete graph edits, coupling atom creation with its first bond so generated
structure can participate in later growth and revision. Paired-graph bridges
provide supervision across alternative construction orders without requiring
recorded edit sequences. Experiments on USPTO-Full and USPTO-50K support the
formulation under predicted and reference reaction-center conditioning.
Controlled ablations and structural analyses connect the gains to growth-aware
editing, missing-fragment construction, and revision of inherited structure.
The sampled paths make these operations explicit within individual predictions,
linking endpoint supervision to the decisions required during generation.

\section*{Reproducibility Statement}

Appendices~\ref{app:representation}--\ref{app:alg-details} detail the representation, evaluation protocol, additional
analyses, qualitative construction paths, and implementation details.

\section*{AI Use Statement}

AI tools assisted with literature discovery, paper-structure review,
code review, and language and \LaTeX{} editing.
The authors performed and verified the experiments, checked the reported
numerical claims, reviewed all AI-assisted material, and remain
responsible for the scientific content.

\begingroup
\normalsize
\bibliography{references}
\bibliographystyle{iclr2027_conference}
\endgroup

\appendix
\renewcommand{\tableformat}{%
  \small
  \setlength{\tabcolsep}{4pt}%
  \renewcommand{\arraystretch}{1.12}%
  \setlength{\aboverulesep}{0.5ex}%
  \setlength{\belowrulesep}{0.5ex}%
}

\section{Representation and reconstruction}
\label{app:representation}

\paragraph{Atom and bond attributes.}
Atom attributes are represented by
\((z,q,h,r,\iota,\chi)\): atomic number, formal charge, explicit-hydrogen count,
radical-electron count, isotope, and atom chirality. Bond attributes are represented by
\((b,s)\): bond type and bond stereochemistry. Aromaticity, conjugation, hybridization,
and implicit valence are not independent generated channels; RDKit derives
them after the explicit graph has been reconstructed.

\paragraph{Atom-order equivalence.}
For a reactant graph \(R\), let
\begin{equation}
  \langle R\rangle=\{\gamma R:\gamma\in\mathfrak S(V_R)\}
  \label{eq:chemical-orbit}
\end{equation}
denote all permutations of atom indices for the same attributed graph. Molecular matching
compares this equivalence class after removing atom maps and canonicalizing
the decoded components.

\paragraph{Atom indices and correspondence.}
Atom mapping identifies corresponding atoms across a reaction;
RXNMapper, for example, infers these correspondences from learned attention
\citep{schwaller2021rxnmapper}. We use the provided atom maps to partition
reactant atoms into inherited product atoms and reactant-only atoms.
The indices of product atoms, \(0,\ldots,|V_P|-1\), are
fixed. Let \(\Sigma_{R\mid P}\) be the permutations that fix these product atoms and
permute only the indices of reactant-only atoms. The target graph represents the
orbit
\begin{equation}
  [R]_P=\{\sigma R:\sigma\in\Sigma_{R\mid P}\}.
  \label{eq:product-relative-orbit-app}
\end{equation}
Terminal decoding removes atom indices and maps every representative of
\([R]_P\) to the same molecular equivalence class \(\langle R\rangle\). In the
quotient bridge, a reactant-only atom receives the next available graph index when created;
the corresponding permutation is applied consistently to the endpoint and all
subsequent bridge targets.

\paragraph{Graph changes relative to the product.}
The representation records the atoms and bonds that differ from \(P\): complete
reactant-only atoms, changed product-atom attributes, missing product bonds, added bonds,
and changed bond fields. Applying these changes to \(P\) creates a labelled
representative of \([R]_P\). Because all explicitly represented atom and bond attributes are
either inherited unchanged or specified by these changes, decoding
and canonicalization recover the same chemical graph whenever RDKit accepts
the explicit valence state. Preprocessing rejects records that fail this
reconstruction check.

In the processed benchmark pairs, the full target reactant set retains the
product atoms, so its total atom count is unchanged or larger. It is larger
for 93.4\% and 92.3\% of evaluated USPTO-50K and USPTO-Full test products,
respectively.

\section{Graph-edit vocabulary and bridge construction}
\label{app:events}

\begin{table}[!htbp]
\centering
\caption{Complete graph edits and their state effects.
A selected edit contains all fields needed to construct its deterministic
successor. Here \(\mathbf m\) and \(\mathbf e\) are complete atom and bond
attribute tuples, and \(d(u)\) is the current degree of a deleted atom.
All edit types preserve unrelated graph records.}
\label{tab:edit-vocabulary}
\tableformat
\begin{tabular}{@{}lllcc@{}}
\toprule
\multirow{2}{*}{Edit type} & \multirow{2}{*}{Location}
 & \multirow{2}{*}{Attributes} & \multicolumn{2}{c}{State change} \\
\cmidrule(lr){4-5}
 & & & \(\Delta|V|\) & \(\Delta|E|\) \\
\midrule
Attached atom addition & existing atom & \(\mathbf m,\mathbf e\) & \(+1\) & \(+1\) \\
Isolated atom addition & graph & \(\mathbf m\) & \(+1\) & \(0\) \\
Generated-atom deletion & most recent generated leaf/isolated \(u\) & none & \(-1\) & \(-d(u)\) \\
\midrule
Bond addition & absent unordered pair & \(\mathbf e\) & \(0\) & \(+1\) \\
Bond deletion & existing bond & none & \(0\) & \(-1\) \\
\midrule
Atom-attribute update & existing atom & replacement \(\mathbf m\) & \(0\) & \(0\) \\
Bond-attribute update & existing bond & replacement \(\mathbf e\) & \(0\) & \(0\) \\
\bottomrule
\end{tabular}
\end{table}

\paragraph{Corrective edits.}
For a recorded reactant graph \(R\), let \(\mathcal C_R(X)\) contain one
correction for each mismatched atom tuple shared by the two graphs, one
deletion for each bond absent from the target graph, one attribute update for each
mismatched bond tuple, and one addition for each missing target-graph bond whose
atoms already exist. Attribute corrections write the complete target tuple.
These edits have unit rate in \(\tau\). On bridge paths, created atoms always
belong to the target graph. Deleting a generated atom has zero target rate. The
training perturbations in Appendix~\ref{app:model-training} can create superfluous
atoms and thereby provide examples of this correction.

\paragraph{Atom-addition distribution.}
Let \(m_R(X)\) be the number of missing reactant-only atoms. The total
atom-addition rate is \(m_R(X)\), distributed according to
\(\rho_R(a\mid X)\) in Equation~\ref{eq:bridge-target}. Write this
rate measure as \(q_R^{\mathrm B}(a\mid X)=m_R(X)\rho_R(a\mid X)\).
For the quotient construction used by the conditioned generator, a missing atom
with a target neighbour proposes an attached atom addition for each such neighbour.
When at least one attached proposal exists,
only attached additions are admitted; otherwise, every missing atom can be
added as an isolated atom. Let \(F_R(X)\) be the resulting eligible atoms and
\(\mathcal B_i(X)\) the complete addition actions for atom \(i\). For
\(m_R(X)>0\), the atom-addition rate is
\begin{equation}
 q_R^{\mathrm B}(a\mid X)=
 \sum_{i\in F_R(X)}\mathbf{1}[a\in\mathcal B_i(X)]
 \frac{m_R(X)}{|F_R(X)|\,|\mathcal B_i(X)|}.
 \label{eq:quotient-rate}
\end{equation}
Each eligible atom receives total rate \(m_R(X)/|F_R(X)|\), distributed
uniformly among its possible attachment atoms. Identical observable actions are merged
by summation. After sampling an atom-addition proposal, the bridge updates the target
correspondence to associate the new node with the chosen target atom and applies the
corresponding fresh-identity permutation to the remaining endpoint. This quotient
operation is not supplied to the rate network.

For either atom-addition distribution, the complete target rate is
\begin{equation}
 q_R(a\mid X)=
 \mathbf{1}[a\in\mathcal C_R(X)]+q_R^{\mathrm B}(a\mid X),
 \label{eq:bridge-target-app}
\end{equation}
where the atom-addition term is zero when no atom is missing. Rates are defined per
unit of \(\tau=-\log(1-t)\); per unit of \(t\), they are multiplied by
\(d\tau/dt=1/(1-t)\).
Let \(D_R(X,\kappa)\) count disagreeing atom and bond tuples, including missing
records, under the current correspondence \(\kappa\), as formalized in
Equation~\ref{eq:deriv-discrepancy}. Every bridge edit reduces this discrepancy.
The target graph is reached after finitely many jumps and is absorbing
for the bridge.

\section{Experimental setup}
\label{app:implementation}

\subsection{Model and training}
\label{app:model-training}

The generator uses a 16-layer graph Transformer with 768-dimensional node
features and 384-dimensional edge features. We implement it in PyTorch with
BF16 mixed-precision training on a single compute node, using distributed
data parallelism for multi-GPU runs. RDKit handles molecular parsing and
reconstruction.

All comparisons use the same graph-Transformer family, complete-edit
representation, and rate-matching objective within each dataset. Models are
trained separately for the two center-information conditions, and the
checkpoint used for testing is selected on validation before the test split is
consulted. The same decoder is used across the center-information conditions.

\paragraph{Bridge states and recovery augmentation.}
Training covers both early and late construction states in transformed time
\(\tau=-\log(1-t)\). When enabled, recovery augmentation perturbs
intermediate states while retaining supervision toward the same reactant
endpoint.

\subsection{Benchmarks and center inputs}
\label{app:benchmark-conventions}

Tables~\ref{tab:full-results} and~\ref{tab:50k-results} use the published
baseline scores and benchmark conventions of RSGPT and RetroDiT
\citep{deng2025rsgpt,wang2026retrodit}.
The USPTO-50K models use the original training split and evaluate
4,944 of 5,007 test products under the ten-new-atom cap.
USPTO-Full evaluation contains 93,289 products.

\paragraph{Reaction-center inputs.}
\label{app:rcset-diagnostics-50k}

Pred. supplies each conditioned trajectory with a predicted center, held fixed
during generation, without reference annotations. Oracle supplies the
reference center information used by the benchmark protocol.

\subsection{Evaluation and decoding}
\label{app:scoring}

Following RetroDiT~\citep{wang2026retrodit}, we compare reactant sets by
atom-and-bond topology and rank distinct decoded candidates by sampling
frequency. Top-$k$ accuracy uses all evaluated products as its denominator.

The accumulated-hazard sampler is specified in Algorithm~\ref{alg:sampling}.
The benchmark uses 100 candidate trajectories and 50 integration intervals per
product, with a decoder budget of at most 128 edits and ten/twenty newly
created atoms on USPTO-50K/Full.

\FloatBarrier
\section{Additional experimental results}
\label{app:additional-results}

\subsection{Training and architecture variants}
\label{app:architecture-variants}

\begin{table}[!htbp]
\centering
\caption{Training and architecture variants. Test Top-$1/3/5/10$
accuracy (\%).}
\label{tab:architecture-ablation-complete}
\tableformat
\begin{tabular}{@{}ll*{8}{r}@{}}
\toprule
Dataset & Variant & \multicolumn{4}{c}{Oracle} & \multicolumn{4}{c}{Pred.} \\
\cmidrule(lr){3-6}\cmidrule(lr){7-10}
 & & 1 & 3 & 5 & 10 & 1 & 3 & 5 & 10 \\
\midrule
\multirow{4}{*}{USPTO-50K}
 & Reference & 72.3 & \textbf{92.9} & \textbf{96.0} & \textbf{97.0} & \textbf{62.4} & 82.8 & \textbf{88.6} & \textbf{91.9} \\
 & Constant addition rate & 72.1 & 89.0 & 91.3 & 93.0 & 61.4 & 80.9 & 85.1 & 88.1 \\
 & Uniform time sampling & 71.6 & 91.5 & 94.5 & 95.8 & 61.2 & 82.3 & 87.4 & 90.5 \\
 & No edge-feature updates & \textbf{73.8} & 92.6 & 95.5 & 96.4 & \textbf{62.4} & \textbf{83.8} & 88.4 & 91.5 \\
\midrule
\multirow{3}{*}{USPTO-Full}
 & Reference & \textbf{68.6} & 80.7 & 82.8 & 84.1 & \textbf{54.3} & \textbf{70.1} & \textbf{74.1} & \textbf{77.1} \\
 & Constant addition rate & 64.9 & 74.9 & 77.2 & 79.3 & 50.0 & 63.6 & 67.1 & 70.1 \\
 & No edge-feature updates & 68.4 & \textbf{80.8} & \textbf{83.1} & \textbf{84.7} & 52.1 & 67.6 & 71.6 & 74.8 \\
\bottomrule
\end{tabular}
\end{table}

\FloatBarrier
\subsection{Supervision ablation}
\label{app:training-controls}

Table~\ref{tab:training-choice-ablation} expands the supervision ablation in
Table~\ref{tab:supervision-main}. The training variants use the same graph
representation, split, and validation protocol. Pred. uses one shared center
predictor; Oracle uses the reference center information.

\begin{table}[!htbp]
\centering
\caption{Supervision comparison on USPTO-50K. Top-$k$ accuracy (\%):
reference on test, training variants on validation. Column maxima are bold.}
\label{tab:training-choice-ablation}
\tableformat
\begin{tabular}{@{}lccrrrr@{}}
\toprule
\multirow{2}{*}{Condition} & \multirow{2}{*}{\shortstack{Partial\\hints}}
 & \multirow{2}{*}{Recovery} & \multicolumn{4}{c}{Top-$k$} \\
\cmidrule(l){4-7}
 & & & 1 & 3 & 5 & 10 \\
\midrule
\multirow{4}{*}{Pred.}
 & \multicolumn{2}{c}{Reference} & 62.4 & 82.8 & \textbf{88.6} & \textbf{91.9} \\
\cmidrule(l){2-7}
 & $\checkmark$ & -- & 62.8 & \textbf{83.6} & 88.0 & 90.4 \\
 & -- & $\checkmark$ & 60.2 & 81.1 & 85.4 & 88.0 \\
 & $\checkmark$ & $\checkmark$ & \textbf{63.3} & 82.8 & 88.2 & 90.2 \\
\midrule
\multirow{4}{*}{Oracle}
 & \multicolumn{2}{c}{Reference} & 72.3 & \textbf{92.9} & \textbf{96.0} & \textbf{97.0} \\
\cmidrule(l){2-7}
 & $\checkmark$ & -- & 72.9 & 91.5 & 93.5 & 94.1 \\
 & -- & $\checkmark$ & \textbf{73.8} & 92.6 & 95.1 & 95.6 \\
 & $\checkmark$ & $\checkmark$ & 73.2 & 91.3 & 94.4 & 95.1 \\
\bottomrule
\end{tabular}
\end{table}

\FloatBarrier
\subsection{Reaction-class analysis}
\label{app:paired-groups}

The reaction-class analysis uses 4,930 paired USPTO-50K products under the same
matching definition and candidate budget as the main evaluation.

\subsubsection{Reaction classes on USPTO-50K}

\begin{table}[!htbp]
\centering
\caption{USPTO-50K accuracy by reaction class. Top-$k$ accuracy (\%) on
4,930 paired products; bold denotes the larger value in each pair. FG denotes
functional-group interconversion.}
\label{tab:reaction-classes}
\tableformat
\begin{tabular}{@{}lrlrrrrrrrr@{}}
\toprule
\multirow{2}{*}{Reaction class} & \multirow{2}{*}{$n$} & \multirow{2}{*}{Method}
 & \multicolumn{4}{c}{Pred.} & \multicolumn{4}{c}{Oracle} \\
\cmidrule(lr){4-7}\cmidrule(lr){8-11}
 & & & Top-1 & Top-3 & Top-5 & Top-10 & Top-1 & Top-3 & Top-5 & Top-10 \\
\midrule
\multirow{2}{*}{Alkyl./aryl.} & \multirow{2}{*}{1,516} & RetroDiT & 65.1 & \textbf{85.4} & 90.6 & \textbf{94.3} & 73.7 & 92.2 & 96.4 & \textbf{98.3} \\
 & & \model{} & \textbf{67.5} & \textbf{85.4} & \textbf{90.8} & 93.6 & \textbf{74.5} & \textbf{94.2} & \textbf{97.6} & 98.0 \\
\addlinespace[1.5pt]
\multirow{2}{*}{Acylation} & \multirow{2}{*}{1,188} & RetroDiT & 74.9 & 91.9 & 94.6 & 96.2 & \textbf{82.2} & \textbf{97.4} & \textbf{98.2} & \textbf{98.7} \\
 & & \model{} & \textbf{75.4} & \textbf{92.8} & \textbf{95.8} & \textbf{97.0} & 81.4 & 97.0 & 98.0 & 98.1 \\
\addlinespace[1.5pt]
\multirow{2}{*}{C--C formation} & \multirow{2}{*}{511} & RetroDiT & \textbf{41.9} & 63.2 & 69.9 & 75.0 & 55.4 & 80.2 & 87.1 & 92.0 \\
 & & \model{} & 40.7 & \textbf{65.9} & \textbf{76.1} & \textbf{82.6} & \textbf{56.2} & \textbf{82.4} & \textbf{88.3} & \textbf{92.4} \\
\addlinespace[1.5pt]
\multirow{2}{*}{Heterocycles} & \multirow{2}{*}{89} & RetroDiT & 6.7 & 7.9 & 7.9 & 7.9 & 69.7 & 79.8 & 82.0 & 84.3 \\
 & & \model{} & \textbf{53.9} & \textbf{71.9} & \textbf{73.0} & \textbf{78.7} & \textbf{77.5} & \textbf{91.0} & \textbf{92.1} & \textbf{95.5} \\
\addlinespace[1.5pt]
\multirow{2}{*}{Protection} & \multirow{2}{*}{67} & RetroDiT & 86.6 & 89.6 & 89.6 & 89.6 & \textbf{95.5} & \textbf{100.0} & \textbf{100.0} & \textbf{100.0} \\
 & & \model{} & \textbf{88.1} & \textbf{92.5} & \textbf{98.5} & \textbf{100.0} & \textbf{95.5} & 97.0 & 97.0 & 97.0 \\
\addlinespace[1.5pt]
\multirow{2}{*}{Deprotection} & \multirow{2}{*}{819} & RetroDiT & \textbf{56.4} & \textbf{84.7} & \textbf{90.6} & \textbf{94.1} & \textbf{59.8} & 88.8 & 93.8 & 96.5 \\
 & & \model{} & 51.2 & 82.2 & 88.8 & 92.7 & 55.2 & \textbf{89.6} & \textbf{94.9} & \textbf{96.8} \\
\addlinespace[1.5pt]
\multirow{2}{*}{Reduction} & \multirow{2}{*}{453} & RetroDiT & 55.2 & 76.2 & 82.1 & 85.7 & 71.3 & 90.3 & 93.4 & 95.1 \\
 & & \model{} & \textbf{61.6} & \textbf{78.4} & \textbf{84.3} & \textbf{89.8} & \textbf{84.8} & \textbf{96.7} & \textbf{97.6} & \textbf{97.8} \\
\addlinespace[1.5pt]
\multirow{2}{*}{Oxidation} & \multirow{2}{*}{82} & RetroDiT & \textbf{69.5} & \textbf{87.8} & 91.5 & 91.5 & \textbf{81.7} & \textbf{93.9} & 95.1 & \textbf{96.3} \\
 & & \model{} & 56.1 & 80.5 & \textbf{92.7} & \textbf{95.1} & \textbf{81.7} & \textbf{93.9} & \textbf{96.3} & \textbf{96.3} \\
\addlinespace[1.5pt]
\multirow{2}{*}{FG interconv.} & \multirow{2}{*}{182} & RetroDiT & \textbf{44.0} & 61.5 & 67.6 & 69.2 & 57.1 & 83.0 & 87.9 & 89.0 \\
 & & \model{} & \textbf{44.0} & \textbf{62.1} & \textbf{70.3} & \textbf{75.3} & \textbf{67.6} & \textbf{87.9} & \textbf{95.1} & \textbf{95.6} \\
\addlinespace[1.5pt]
\multirow{2}{*}{FG addition} & \multirow{2}{*}{23} & RetroDiT & \textbf{82.6} & \textbf{91.3} & \textbf{91.3} & 91.3 & 78.3 & 91.3 & 91.3 & 95.7 \\
 & & \model{} & \textbf{82.6} & \textbf{91.3} & \textbf{91.3} & \textbf{100.0} & \textbf{95.7} & \textbf{100.0} & \textbf{100.0} & \textbf{100.0} \\
\bottomrule
\end{tabular}
\end{table}

\clearpage
\section{Qualitative construction paths}
\label{app:qualitative-examples}

\Needspace{8\baselineskip}
\subsection{Reactant construction examples}
\label{app:construction-cases}

Figure~\ref{fig:full-construction-examples} shows three further USPTO-Full
Oracle construction examples.

\begin{figure}[!htbp]
\centering
\begin{subfigure}[t]{0.32\linewidth}
\centering
\includegraphics[height=115pt,width=\linewidth,keepaspectratio]{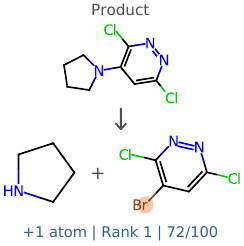}
\caption{Leaving-group recovery.}
\label{fig:full-case-a}
\end{subfigure}\hfill
\begin{subfigure}[t]{0.32\linewidth}
\centering
\includegraphics[height=115pt,width=\linewidth,keepaspectratio]{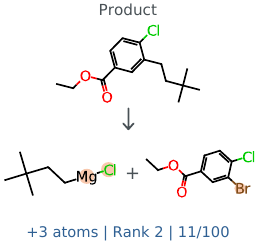}
\caption{Coupling-partner recovery.}
\label{fig:full-case-b}
\end{subfigure}\hfill
\begin{subfigure}[t]{0.32\linewidth}
\centering
\includegraphics[height=115pt,width=\linewidth,keepaspectratio]{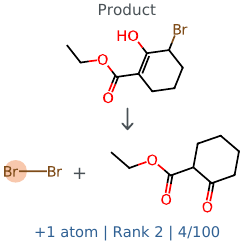}
\caption{Bond-order changes.}
\label{fig:full-case-c}
\end{subfigure}
\caption{Reactant construction on USPTO-Full under Oracle conditioning.
Added atoms are orange; the panels illustrate fragment recovery, coupling
partner recovery, and bond-order revision.}
\label{fig:full-construction-examples}
\end{figure}

Panel (a) combines one added bromine atom with separation of the cyclic
amine from the heteroaromatic scaffold. Panel (b) requires three added atoms
distributed across the two reactants, including an organomagnesium coupling partner.
Panel (c) combines one added atom with ring bond-order changes.

\FloatBarrier
\Needspace{8\baselineskip}
\subsection{Sampled construction paths}
\label{app:sampled-trajectories}

Figures~\ref{fig:trajectory-50k-a}--\ref{fig:trajectory-full} show complete
sampled construction paths under Oracle conditioning. Each example includes
the product and the successive edits; waiting intervals with no graph change
are omitted. These are construction states, not chemical reaction mechanisms.

\paragraph{Chemical reaction classes on USPTO-50K.}
Figures~\ref{fig:trajectory-50k-a} and~\ref{fig:trajectory-50k-b} cover the
ten USPTO-50K reaction classes. The paths range from a single atom addition
to 14 edits, including trajectories that interleave fragment growth with
atom and bond updates.

\begin{figure}[tb!]
\centering
\includegraphics[width=\linewidth]{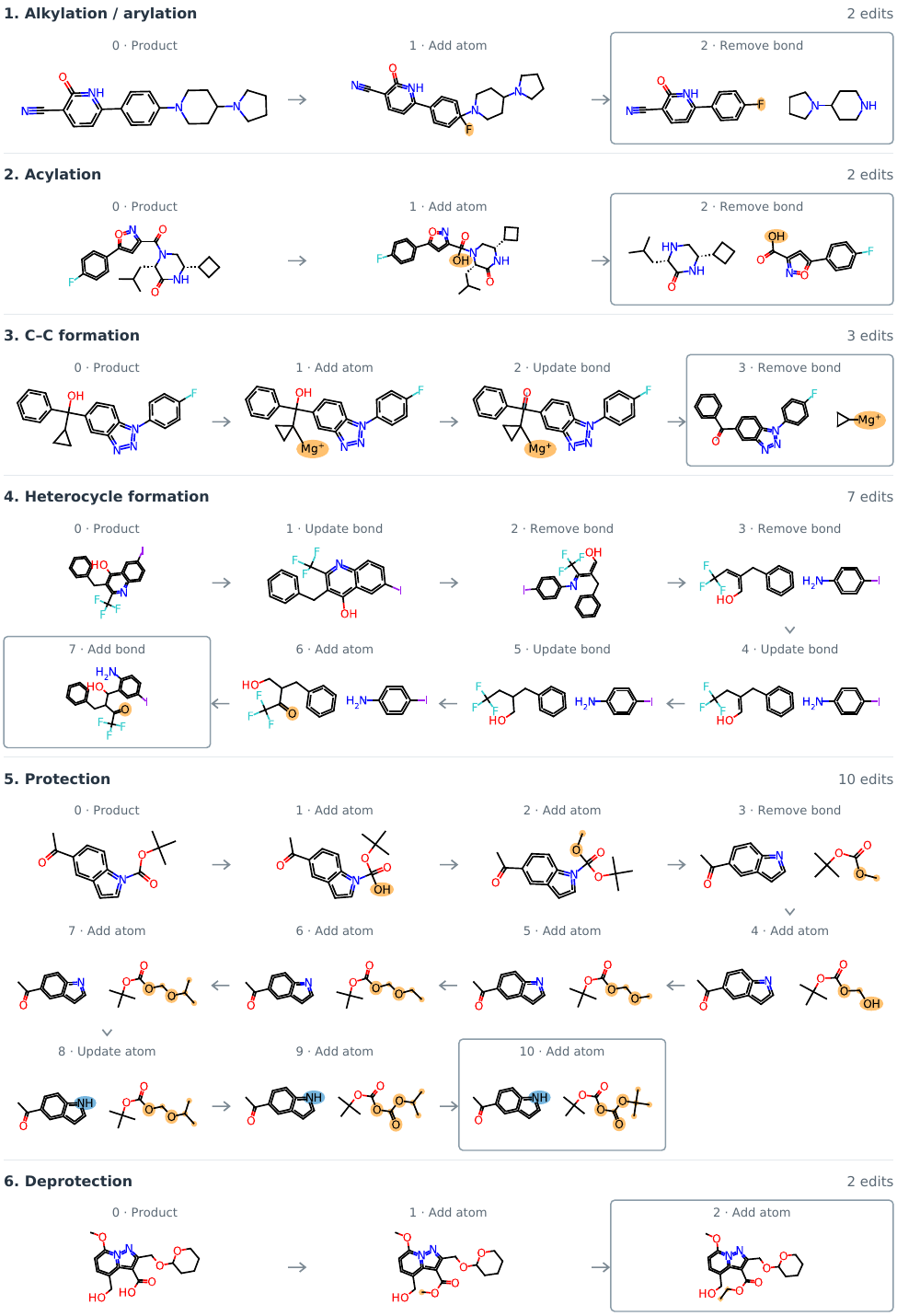}
\caption{Complete trajectories for USPTO-50K classes 1--6.
Orange marks added atoms and their bonds; blue marks changes to inherited
atoms or bonds. Every arrow represents one edit.}
\label{fig:trajectory-50k-a}
\end{figure}

\begin{figure}[tb!]
\centering
\includegraphics[width=\linewidth]{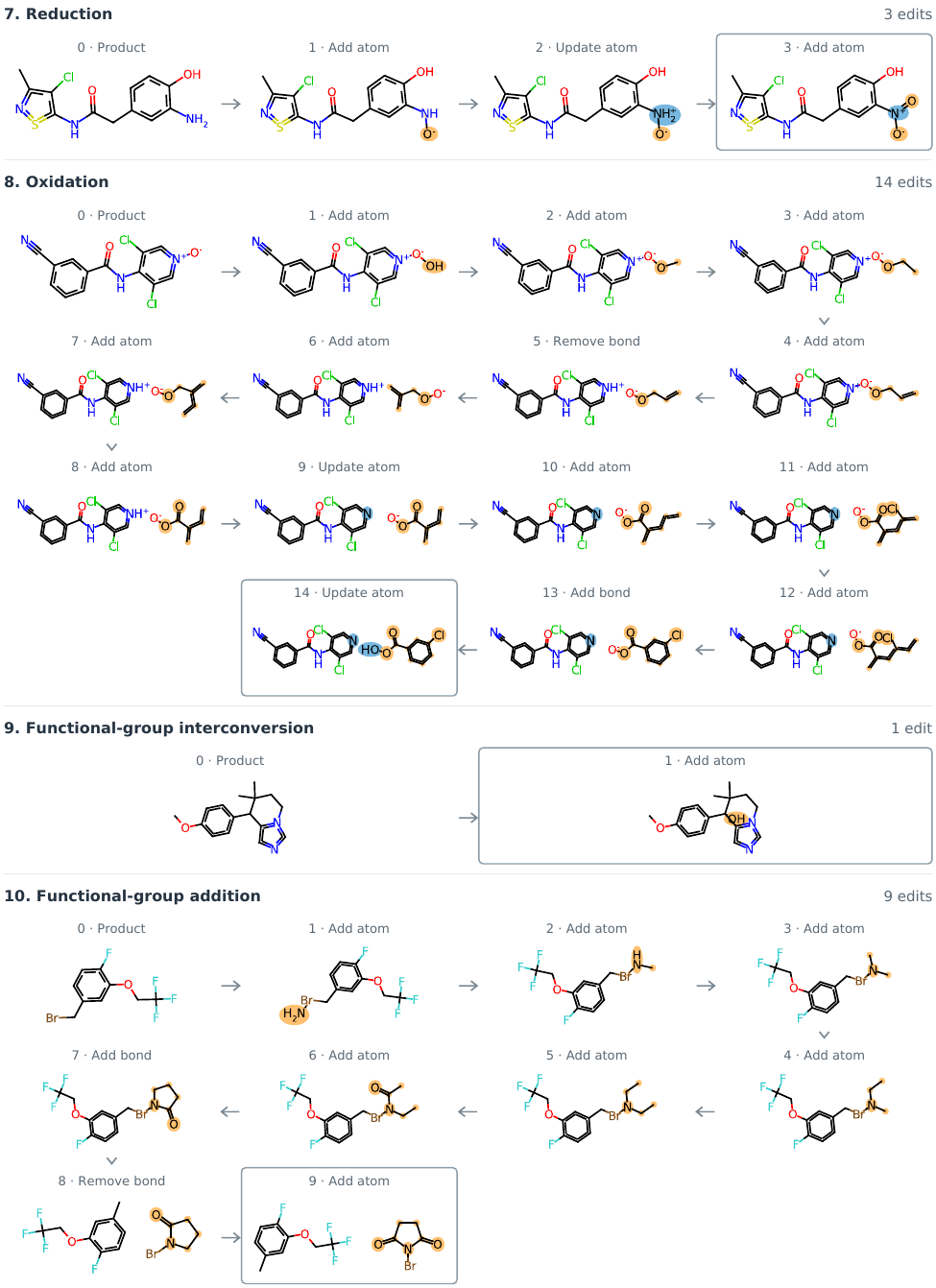}
\caption{Complete trajectories for USPTO-50K classes 7--10.
The oxidation example shows all 14 edits, including ten atom additions.}
\label{fig:trajectory-50k-b}
\end{figure}

\paragraph{Graph-edit families on USPTO-Full.}
The USPTO-Full records used here have no explicit chemical reaction-class
labels. Figure~\ref{fig:trajectory-full} instead groups examples by their
graph changes. The first path starts from a 48-atom product and reaches
a 52-atom reactant set in 14 edits: four atom additions, eight bond-attribute
updates, one atom-attribute update, and one bond removal. The other two
retain the product atom count while updating
atom attributes and bonds, or bond attributes alone.

\begin{figure}[tb!]
\centering
\includegraphics[width=\linewidth]{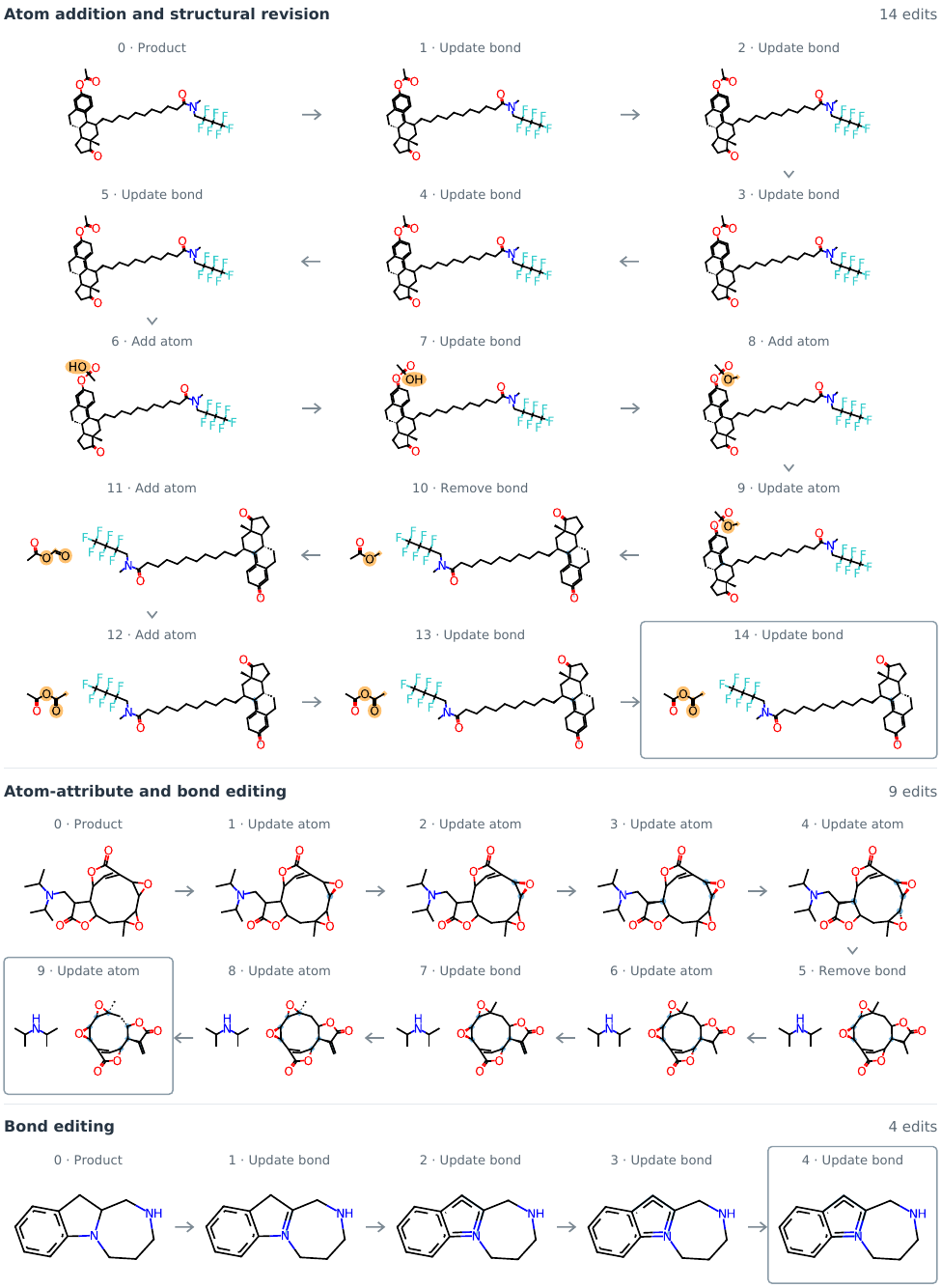}
\caption{Complete trajectories across USPTO-Full graph-edit families.
The three paths contain 14, nine, and four edits; each path is shown on its
recorded inference grid.}
\label{fig:trajectory-full}
\end{figure}

\section{Rate matching and bridge marginals}
\label{app:generator}

The target bridge is Markov on a state that includes more information than the
rate network observes. Besides the current graph \(X_t\), it retains the
recorded reactant graph \(R\) and a correspondence \(\kappa_t\) from current
atoms to target atoms. Write \(U_t=(R,\kappa_t)\) for this hidden information.
The notation \(q_R\) in the main text suppresses the correspondence. Throughout
this appendix, the product and any supplied reaction-center condition are
fixed visible conditioning variables. We suppress the center argument in the
formulas, including conditional expectations. Unadorned rates \(q\) and
\(r\) are per unit of transformed time \(\tau\).

We state the results for a finite representable graph state space, obtained with
finite attribute vocabularies, an atom cap, and a finite set of atom indices.
Targets in the conditional distribution and their bridge edits are represented
in this space; endpoint validity is checked during decoding.

\subsection{From latent proposals to graph transitions}

A complete action specifies an observable edit, whereas a bridge proposal can
also specify which target atom it realizes. Let \(\mathcal H_U(X)\) be the
finite proposal set, \(w_U(h\mid X)\) its nonnegative rates, and
\(o_X(h)\) the complete action obtained by omitting the proposal's hidden
target assignment. The action rate is the pushforward of the proposal rates:
\begin{equation}
 q_U(a\mid X)=
 \sum_{h\in\mathcal H_U(X):\,o_X(h)=a}w_U(h\mid X).
 \label{eq:deriv-action-aggregation}
\end{equation}
For quotient atom addition, different missing target atoms can therefore
contribute to the same attached atom-addition action. If a
simulator samples the merged action first, it recovers a compatible latent
proposal with probabilities \(w_U(h\mid X)/q_U(a\mid X)\). This second draw
updates the correspondence used by subsequent bridge steps.

Actions can be aggregated once more when several complete edits have the same
successor. For \(Y\ne X\), the graph transition rate and its diagonal entry are
\begin{equation}
 Q_U^\tau(X,Y)=\sum_{a:T_aX=Y}q_U(a\mid X),\qquad
 Q_U^\tau(X,X)=-\sum_{Y\ne X}Q_U^\tau(X,Y).
 \label{eq:deriv-successor-rates}
\end{equation}
We take actions to be genuine edits, so \(T_aX\ne X\). Null actions, if present
in a proposal representation, contribute nothing to a graph generator and can
be omitted. The distinction between proposals, actions, and successors avoids
treating an unobserved target assignment as an extra graph transition.

Replacing target rates by learned rates gives the generator acting on any
real-valued function \(f\) of the graph:
\begin{equation}
 (\mathcal Q_\theta^\tau f)(X;P)=
 \sum_{a\in\mathcal A(X)}r_\theta(a\mid X,P,t(\tau))
 \bigl[f(T_aX)-f(X)\bigr].
 \label{eq:event-generator}
\end{equation}
The action distribution is useful computationally, but the sums in
Equation~\ref{eq:deriv-successor-rates} determine the evolution of graph
probabilities. Learning all action rates is sufficient to learn these sums.

\subsection{A decreasing discrepancy and finite absorption}

Fix a target \(R\). On an unperturbed bridge, every current atom corresponds
to a distinct target atom, with product-atom correspondences fixed. Existing assignments are
never revised. An atom addition extends \(\kappa\) to one previously missing target
atom. Equivalently, each realized path can be compared with its target using
one consistent set of target atom indices.

These shared indices make discrepancies involving uncreated atoms
well-defined. Let \(\nu_X^\kappa(i)\) be the current atom tuple at target
index \(i\), or a missing-node symbol if that atom does not yet exist.
Let \(\beta_X^\kappa(i,j)\) be the current bond tuple, with a no-bond symbol
when a bond or either endpoint is absent. Define
\begin{equation}
 D_R(X,\kappa)=
 \sum_{i\in V_R}\mathbf 1[\nu_X^\kappa(i)\ne\nu_R(i)]
 +\sum_{\{i,j\}\subset V_R}
     \mathbf 1[\beta_X^\kappa(i,j)\ne\beta_R(i,j)].
 \label{eq:deriv-discrepancy}
\end{equation}
The second sum counts missing target bonds even before their endpoints are
created. Counting only discrepancies among currently existing atoms would not
give the required monotonicity: an atom addition could expose additional missing bonds.

\begin{proposition}[Absorption of a target bridge]
\label{prop:deriv-absorption}
Suppose the target contains all product atoms, agrees on their immutable
attributes, and the atom-addition and other corrective actions described in
Appendix~\ref{app:events} are admissible. Starting at \(P\), every bridge jump
strictly decreases \(D_R\). The bridge reaches a representative of \([R]_P\)
after at most \(D_R(P,\kappa_0)\) jumps and in finite transformed time almost
surely. That terminal state is absorbing for the target bridge.
\end{proposition}

\begin{proof}
A corrective edit writes a disagreeing record to its target value; an attached
atom addition also writes its incident target bond. No correct record is changed,
so the nonnegative integer \(D_R\) decreases at every jump. Before absorption,
the bridge rate \(\Lambda_U(X)=m_R(X)+|\mathcal C_R(X)|\) is finite and at
least one. Hence at most \(D_R(P,\kappa_0)\) exponential holding times are
needed, their sum is finite almost surely, and both rate terms vanish at the
terminal representative.
\end{proof}

The argument permits different construction orders, including the two in
Figure~\ref{fig:edit-paths}, without requiring intermediate graphs to satisfy
molecular valence rules.

\subsection{Generalized KL separates intensity from edit choice}
\label{app:rate-decomposition}

For finite nonnegative action-rate measures \(q\) and \(r\), define
\begin{equation}
 D_{\rm GKL}(q\Vert r)=
 \sum_a\left[q(a)\log\frac{q(a)}{r(a)}-q(a)+r(a)\right].
 \label{eq:deriv-rate-divergence}
\end{equation}
We use the conventions \(0\log(0/r)=0\) and infinite divergence when
\(q(a)>0=r(a)\). Let \(\Lambda=\sum_aq(a)\) and
\(\lambda=\sum_ar(a)\). When both are positive, set
\(\nu=q/\Lambda\) and \(\pi=r/\lambda\). Substituting these factorizations
and using \(\sum_a\nu(a)=1\) gives
\begin{equation}
 D_{\rm GKL}(q\Vert r)=
 \underbrace{\Lambda\log\frac{\Lambda}{\lambda}-\Lambda+\lambda}
 _{\text{intensity divergence}}
 +\Lambda D_{\rm KL}(\nu\Vert\pi).
 \label{eq:deriv-kl-decomposition}
\end{equation}
Thus action-choice errors are weighted by the target edit intensity. Removing
terms independent of the model leaves
\(\lambda-\Lambda\log\lambda-\sum_aq(a)\log\pi(a)\), exactly the loss in
Equation~\ref{eq:gkl}. When \(\Lambda=0\), the divergence is simply
\(\lambda\): an absorbed target supplies intensity supervision but no preferred
next action. The optimum lies at zero intensity, a boundary approached rather
than attained by a finite softplus output.

Scaling both rate measures by \(c>0\) scales their divergence by \(c\).
Consequently using rates per unit of \(t\) would multiply the statewise loss
by \(1/(1-t)\). The manuscript instead matches transformed-time rates.
Different positive time-sampling weights preserve the unrestricted
statewise optimum, but need not produce the same finite-capacity fit.

\subsection{Optimal prediction with a hidden target and correspondence}

Let \(Z=(X_t,P,t)\) denote the visible conditioning, including any supplied
center as above. Define \(\bar q_Z(a)=\mathbb E[q_{U_t}(a\mid X_t)\mid Z]\).
All latent alternatives are expressed on the same visible action set, with
zero rate assigned to actions absent from a particular target's support.

\begin{proposition}[Conditional-mean rate]
\label{prop:deriv-conditional-rate}
For a finite action set and finite conditional expectations of \(q(a)\) and
\(q(a)\log q(a)\), the minimizer
of conditional expected generalized KL over nonnegative rate measures is
\(r^*(a\mid Z)=\bar q_Z(a)\).
\end{proposition}

\begin{proof}
For each action, the model-dependent conditional risk is
\(r(a)-\bar q_Z(a)\log r(a)\). If \(\bar q_Z(a)>0\), its derivative is
\(1-\bar q_Z(a)/r(a)\), with a unique minimum at \(r(a)=\bar q_Z(a)\).
If the mean is zero, the risk is \(r(a)\), minimized at zero. Equivalently,
direct expansion gives the excess-risk identity
\begin{equation}
 \mathbb E[D_{\rm GKL}(q_{U_t}\Vert r)\mid Z]
 =\mathbb E[D_{\rm GKL}(q_{U_t}\Vert\bar q_Z)\mid Z]
  +D_{\rm GKL}(\bar q_Z\Vert r).
 \label{eq:deriv-excess-risk}
\end{equation}
\end{proof}

The first term in Equation~\ref{eq:deriv-excess-risk} reflects uncertainty
about the hidden target and correspondence even when the visible-state rate
is predicted optimally. When \(\bar\Lambda_Z=\mathbb E[\Lambda_{U_t}\mid Z]>0\),
the optimal factorization is
\begin{equation}
 \lambda^*(Z)=\bar\Lambda_Z,\qquad
 \pi^*(a\mid Z)=
 \frac{\mathbb E[q_{U_t}(a\mid X_t)\mid Z]}{\bar\Lambda_Z}.
 \label{eq:deriv-optimal-factorization}
\end{equation}
This is an intensity-weighted average of target action distributions, not
their unweighted average.

\subsection{Consistency of one-time graph marginals}
\label{app:marginal-consistency}

The generator-matching argument \citep{holderrieth2024generatormatching}
applies to the visible projection even when that projection is not itself
Markov. Let \(p_\tau(x)\) be the conditional bridge law given the visible
inputs, with locally integrable rates on the finite state space.

\begin{proposition}[Graph-edit marginal consistency]
\label{prop:deriv-marginals}
\label{prop:main-marginal}
Fix the product and any supplied center condition. Assume the unperturbed
bridges have a finite admissible state space containing their targets,
consistently extended atom correspondences, and locally integrable rates
defining non-explosive processes. Over unrestricted nonnegative rates,
the population optimum of Equation~\ref{eq:gkl}, per unit \(\tau\), is
the conditional mean in Equation~\ref{eq:main-marginal-rate}.
The expectation includes hidden atom correspondences. The time-inhomogeneous
Markov process initialized at \(P\) with these rates has marginal law
\(p_\tau\) at every finite \(\tau\).
\end{proposition}

\begin{proof}
Proposition~\ref{prop:deriv-conditional-rate} gives the conditional-mean
optimum in Equation~\ref{eq:main-marginal-rate}. For any graph function
\(f\), conditioning the bridge generator on the visible
state replaces \(q_{U_t}\) by its conditional mean \(r^*\). The conditional
bridge law and the Markov process with rates \(r^*\) therefore satisfy the same
finite forward equation and initial condition; uniqueness gives equal one-time
marginals.
\end{proof}

Under Proposition~\ref{prop:deriv-absorption}, each bridge eventually reaches
its own target representative. For targets accepted by the reconstruction
check in Appendix~\ref{app:representation}, decoding maps the limiting
marginal to the recorded-target distribution used to construct the bridges.
The learned sampler approximates this population process with a fitted rate
network, a finite horizon, and numerical integration.

\subsection{Perturbed training states and numerical generation}

State perturbations augment the training population with recoverable graph
errors. If \(\widetilde U_t\) retains the target and its correspondence with a
perturbed graph \(\widetilde X_t\), recomputing corrective rates gives
\begin{equation}
 \widetilde r^*(a\mid\widetilde X_t,P,t)=
 \mathbb E_{\rm aug}[q_{\widetilde U_t}(a\mid\widetilde X_t)
                    \mid\widetilde X_t,P,t],
 \label{eq:deriv-augmented-rate}
\end{equation}
which supervises recovery without assigning a unique correction order.

\section{Algorithms and state evolution}
\label{app:alg-details}

\subsection{Graph state and atom indices}
\label{app:alg-state}

The initial state is \(P\). A new atom receives the largest current atom index
plus one; subsequent edits can revise its attributes or incident bonds.
Surviving atoms retain their indices after a deletion.

Attribute tuples and complete edits follow Appendices~\ref{app:representation}
and \ref{app:events}. Bond absence is represented by a missing edge.

The state update can delete an isolated or leaf generated node; the rate network
restricts this action to the most recently added generated node (the highest
allocated atom index among generated atoms).
A leaf deletion removes its incident bond. Atom-addition masks enforce the
new-atom cap.

\subsection{Constructing bridge states and sparse rate targets}
\label{app:alg-training}

During a quotient bridge, the endpoint indices assigned to uncreated atoms
can change as atoms are added.
The endpoint is relabelled to extend the correspondence with the newly
created atom, while existing state identities and assignments remain unchanged.
Denote the endpoint together with this correspondence by \(U\). The model never
receives \(U\) or a fixed slot assignment.

The bridge first forms weighted proposals \((a,w,i)\), where \(a\) is a
complete observable edit and \(i\) is the endpoint atom represented by an
atom addition, or an empty tag for other corrections. Proposal weights follow
Appendix~\ref{app:events}. An attached atom addition supplies only one
first bond. Other endpoint bonds incident to the new node become eligible
bond-addition corrections once both of their atoms exist. In perturbed states,
corrections can also remove a superfluous newest isolated or leaf atom.

Distinct hidden endpoint atoms can propose the same observable action. Bridge
simulation samples the tagged proposals, retaining the selected endpoint
identity for the next correspondence update. The training target instead
merges them:
\begin{equation}
  q_U(a\mid X)=\sum_{(a',w,i):\,a'=a}w.
  \label{eq:alg-merged-target}
\end{equation}
Thus the learner is not asked to distinguish endpoint assignments that yield
the same current edit. The total rate is unchanged by this summation.
Algorithm~\ref{alg:training} uses exponential waiting times to simulate the
bridge \citep{gillespie1977simulation}.

\begin{algorithm}[!ht]
\caption{Bridge supervision for one training pair}
\label{alg:training}
\small
\begin{algorithmic}[1]
\Require Product $P$, reactants $R$, rate network $\theta$
\Ensure Rate-matching loss for one sampled state
\State Sample $t_\star$ from the training time distribution; $\tau_\star\gets-\log(1-t_\star)$
\State $X\gets P$, $\tau\gets0$; initialize endpoint correspondence $U$ from $(P,R)$
\While{$\tau<\tau_\star$}
  \State $\mathcal Q\gets\Call{TaggedCorrections}{X,U}$; $\Lambda\gets\sum_{(a,w,i)\in\mathcal Q}w$
  \If{$\Lambda=0$} \State \textbf{break} \EndIf
  \State Draw $\Delta\tau\sim\operatorname{Exp}(\Lambda)$
  \If{$\tau+\Delta\tau>\tau_\star$} \State \textbf{break} \EndIf
  \State Draw $(a,w,i)\in\mathcal Q$ with probability $w/\Lambda$
  \State $X\gets T_aX$; $\tau\gets\tau+\Delta\tau$
  \State Update $U$ using $i$ if $a$ adds an atom
\EndWhile
\State Optionally apply recoverable perturbations to $X$
\State Recompute $\mathcal Q$; merge identical actions into $q_U$ using Eq.~\ref{eq:alg-merged-target}
\State $(\lambda_\theta,\pi_\theta)\gets\Call{RateNetwork}{P,X,t_\star}$
\State \Return $\mathcal L_{\mathrm{rate}}(\lambda_\theta,\pi_\theta;q_U)$ from Eq.~\ref{eq:gkl}
\end{algorithmic}
\end{algorithm}

Optimization averages the loss in Algorithm~\ref{alg:training} over the minibatch.

\subsection{Factorization of edit probabilities}
\label{app:alg-factorization}

The probability of a complete edit is evaluated through a hierarchy of normalized
conditional distributions. Let \(Z=(P,X,t)\) denote the visible context.
The graph readout predicts total intensity and edit-type probabilities.
Type-specific distributions then select locations and attributes. For an
attached atom addition at existing atom \(j\), the model jointly selects \(j\)
and the new element \(z\). The remaining atom attributes are generated
autoregressively in tuple order. Writing \(f_{\mathrm A}\) for this
edit type as in the main text, the factorization is
\begin{align}
 \pi_\theta(f_{\mathrm A},j,\mathbf m,\mathbf e\mid Z)
 &=\pi_F(f_{\mathrm A}\mid Z)
       \pi_{J,z}(j,z\mid f_{\mathrm A},Z)
 \nonumber\\
 &\quad\times\prod_{k=2}^{6}\pi_k(m_k\mid m_{<k},j,Z)
       \pi_E(\mathbf e\mid\mathbf m,j,Z).
 \label{eq:alg-attached-factorization}
\end{align}
\begin{samepage}
The new atom's index is determined by the current state and adds no
categorical choice. The autoregressive atom decoder updates its hidden state
between attributes, so later attributes depend on the earlier choices.
For attached atom additions, the first-bond prediction is conditioned on all
attributes of the new atom. The bond decoder jointly predicts bond type and
stereochemistry. This factorization couples attachment, atom attributes,
and first-bond attributes within one complete edit.\par
\end{samepage}

Isolated atom additions use a graph-level element distribution followed by the same
sequence of attribute predictions, without a bond prediction. They introduce only
one atom; subsequent attached atom additions or bond edits determine the component's
eventual structure. Bond additions select
an absent unordered pair and then jointly predict its bond attributes. Bond-attribute updates select an
existing bond and mask its current tuple out of the replacement distribution.
For updates that keep the element fixed, a nonempty mask selects the atom
attributes to change. Replacement values exclude the current values, while untouched
fields are copied into the complete tuple. Deletions specify locations without new attributes.

Masks are applied before the corresponding normalization. An edit type with no
eligible location is excluded from the type distribution. During atom addition,
the allowed values of each attribute can depend on earlier choices, and
the allowed bond attributes can depend on the new atom tuple. Training computes
the log probability of the complete target action by adding these conditional
log probabilities. Generation samples actions through the same hierarchy.

\subsection{Generation by accumulated hazard}
\label{app:alg-sampling}

The numerical sampler partitions transformed time uniformly up to
\(\tau_{\mathrm{end}}=-\log(1-t_{\mathrm{end}})\). The network still
receives \(t=1-e^{-\tau}\). At a rate evaluation, it freezes both total
intensity \(\widehat\lambda\) and action distribution \(\widehat\pi\)
until the next grid boundary or a graph edit. If the next boundary is
\(\beta\), the integrated hazard over an unchanged interval is approximated
by
\begin{equation}
 \int_\tau^\beta
 \lambda_\theta(X,P,1-e^{-s})\,ds
 \;\approx\;\widehat\lambda(\beta-\tau).
 \label{eq:alg-interval-hazard}
\end{equation}
An exponential threshold is consumed across intervals. Crossing a grid
boundary changes the rate approximation, but does not restart the event
clock. When the threshold is crossed within an interval, the action is sampled
from the same frozen distribution used for that event's waiting time.

\begin{algorithm}[!ht]
\caption{Generation with accumulated hazard}
\label{alg:sampling}
\small
\begin{algorithmic}[1]
\Require Product $P$, rates $\theta$, grid ending at $\tau_{\mathrm{end}}$, caps $N_{\mathrm{edit}},N_{\mathrm{new}}$
\Ensure Terminal graph $X$ or a budget-limited terminal state
\State $X\gets P$, $\tau\gets0$, $n\gets0$; draw $H\sim\operatorname{Exp}(1)$
\While{$\tau<\tau_{\mathrm{end}}$}
  \State $\beta\gets$ next grid boundary after $\tau$
  \State $(\widehat\lambda,\widehat\pi)\gets\Call{RateNetwork}{P,X,1-e^{-\tau}}$ with current masks and $N_{\mathrm{new}}$
  \State $h\gets\widehat\lambda(\beta-\tau)$
  \If{$\widehat\lambda=0$ \textbf{or} $H>h$}
    \State $H\gets H-h$; $\tau\gets\beta$ \Comment{Keep the event clock}
  \Else
    \If{$n=N_{\mathrm{edit}}$} \State \Return \textsc{BudgetLimit} \EndIf
    \State $\tau\gets\tau+H/\widehat\lambda$; draw $a\sim\widehat\pi$
    \State $X\gets T_aX$; $n\gets n+1$
    \State Draw $H\sim\operatorname{Exp}(1)$ \Comment{New clock after the edit}
  \EndIf
\EndWhile
\State \Return $X$
\end{algorithmic}
\end{algorithm}

Algorithm~\ref{alg:sampling} re-evaluates the rates after each edit and each
grid crossing. The grid controls integration accuracy, not the edit count.

Updating the graph immediately matters for growth: a newly created atom can
provide an attachment site, form another bond, or undergo an attribute update
before the next grid boundary. Conversely, several boundaries can pass without
an edit. The new-atom cap disables atom addition while retaining other admissible
edit types. If the edit budget is reached before another edit can be applied,
sampling returns the current state; otherwise it returns the graph at the
terminal time.

\subsection{From terminal states to canonical reactant sets}
\label{app:alg-decoding}

Decoding omits internal atom indices. Full reconstruction writes explicit
fields and bond stereochemistry, then uses RDKit sanitization and stereochemistry
assignment to obtain canonical isomeric SMILES \citep{weininger1988smiles}. For
the primary topology-matching score, decoding retains elements and core bond
types as non-isomeric SMILES. Each view decodes independently.

Sorted component strings form a canonical key, retaining repeated components.
Evaluation follows Appendix~\ref{app:scoring}.

\FloatBarrier
\section{Limitations and Future Work}
\label{app:limitations}

\model{} studies single-step reactant construction from paired reaction
records, with evaluation on patent-derived benchmarks. Extending this setting
to more diverse reaction collections and chemical domains is a natural
direction for examining how the learned construction strategies transfer
across molecular families. Broader evaluation could also complement recovery
of recorded reactants with assessments of alternative precursors under
different synthesis objectives. Such studies would connect the standardized
comparisons used here with a wider range of retrosynthetic choices, including
transformations and applications beyond those represented in the current
benchmarks.

A second direction is to connect dynamic graph construction with
condition-aware generation and multistep synthesis planning. Reaction
conditions, precursor availability, and route-level objectives could provide
additional context for choosing both local edits and complete reactant sets.
Incorporating these factors would allow the single-step generator to be
studied as part of an overall synthesis strategy, where the value of a
prediction depends on the route it enables. Computational assessment and
experimental feedback could help evaluate these candidates in practical
synthesis settings.

\end{document}